\pdfoutput=1

\documentclass{article}

\usepackage{arxiv}

\usepackage{amssymb,amsfonts, amsmath}
\usepackage{clrscode3e}
\usepackage{amsthm} 
\usepackage[linesnumbered,ruled,vlined,noend]{algorithm2e}

\usepackage[utf8]{inputenc}
\usepackage{subcaption}
\usepackage{xspace}
\usepackage{graphicx}
\usepackage{textcomp}
\usepackage{xcolor}
\usepackage{multirow}

\newtheorem{Definition}{Definition}

\newtheorem{Theorem}{Theorem}
\newtheorem{Lemma}{Lemma}
 
\newtheorem{Corollary}{Corollary}

\newcommand{\BO}[1]{O\left(#1\right)}
\newcommand{\BOM}[1]{\Omega\left(#1\right)}
\newcommand{\BT}[1]{\Theta\left(#1\right)}

\newcommand{\TTL}{\mbox{\rm TTL}}

\newcommand{\distn}{\mbox{\rm dist}}
\newcommand{\mindist}{d_{\rm min}}
\newcommand{\maxdist}{d_{\rm max}}

\newcommand{\gon}{\textsc{gon}\xspace}

\renewcommand{\epsilon}{\varepsilon}
\newcommand{\charikar}{\textsc{charikar}\xspace}
\newcommand{\sampledCharikar}{\textsc{samp-charikar}}
\newcommand{\coresetOutliers}{\textsc{our-sliding}}
\newcommand{\coresetEf}{\textsc{eff-sliding}}
\newcommand{\seqEf}{\textsc{eff-sequential}}

\renewcommand{\arraystretch}{1.1}

\DeclareMathOperator*{\argmax}{argmax} %
\DeclareMathOperator*{\argmin}{argmin} %

\title{$k$-Center Clustering with Outliers in Sliding Windows}

\author{Paolo Pellizzoni\\
University of Padova\\
 Padova,  Italy\\
\texttt{paolo.pellizzoni@studenti.unipd.it}\\
\And
Andrea Pietracaprina\\
University of Padova\\
 Padova,  Italy\\
\texttt{andrea.pietracaprina@unipd.it}\\
\And
Geppino Pucci\\
University of Padova\\
 Padova,  Italy\\
\texttt{geppino.pucci@unipd.it}\\
}

\begin{document}
\maketitle
\begin{abstract}
Metric $k$-center clustering is a fundamental unsupervised learning
primitive. Although widely used, this primitive is heavily affected by
noise in the data, so that a more sensible variant seeks for the best
solution that disregards a given number $z$ of points of the dataset,
called outliers. We provide efficient algorithms for
this important variant in the streaming model under the sliding window
setting, where, at each time step, the dataset to be clustered is the
window $W$ of the most recent data items. Our algorithms achieve
$\BO{1}$ approximation 
and, remarkably, require a working memory linear in $k+z$ and
only logarithmic in $|W|$. 
As a by-product, we show how to estimate the effective diameter
of the window $W$, which is a measure of the spread of the window
points, disregarding a given fraction of noisy distances. We also
provide experimental evidence of the practical viability of our
theoretical results.
\end{abstract}


\
\section{Introduction}
In a number of modern scenarios (e.g., social network analysis, online
finance, online transaction processing, etc.), data are produced as a
continuous stream, and at such a high rate that on-the-fly processing
can only afford to maintain a small portion of the data in memory,
together with a limited amount of working space.  This computational
scenario is captured by the well-established \emph{streaming model}
\cite{HenzingerRR98}. The \emph{sliding window setting}
\cite{DatarM16, Braverman16}, introduces the additional, desirable
constraint that the input for the problem of interest consists of the
window $W$ of the most recent data items, while older data are
considered ``stale'' and disregarded by the computation.

The $k$-center clustering problem (\emph{$k$-center}, for short) is a
fundamental unsupervised learning primitive with ubiquitous
applications \cite{Snyder11,HennigMMR15,BateniEFM21}.  Given a set $W$ of
points from a metric space, the $k$-center problem requires
determining a subset $C \subset W$ of $k$ centers which minimize the
maximum distance of any point of $W$ from its closest center. However,
since the objective function involves a maximum, the solution is at
risk of being severely influenced by a few ``distant'' points,
called \emph{outliers}.  In fact, the presence of outliers is inherent
in many large datasets, since these points are often artifacts of data
collection, either representing noisy measurements or simply erroneous
information.  To cope with this limitation, $k$-center admits a
heavily studied robust formulation that takes into account
outliers \cite{Charikar2001,MalkomesKCWM15,CeccarelloPP19}: when
computing the objective function for a set of $k$ centers, the $z$
largest distances from the centers are to be discarded, where $z<|W|$
is an additional input parameter representing the tolerable level of
noise. This formulation is known as the \emph{$k$-center problem with
$z$ outliers}.

The decision versions of the $k$-center problem and its variant with $z$ outliers are NP-complete  \cite{GONZALEZ1985293}, 
hence only approximate solutions may be returned within reasonable time bounds. 
In this paper, we present approximation algorithms for the $k$-center
problem with $z$ outliers in the sliding window setting. Moreover, as a by product, we also derive
an algorithm to estimate the $\alpha$-effective diameter of the
window, which is an interesting measure of the spread of a noisy
dataset.

\subsection{Related Work}
In the sequential setting, the $k$-center problem (without
outliers) admits simple 2-approximation
algorithms \cite{GONZALEZ1985293,HochbaumS85}. Recently, a sequential, randomized,  fully-dynamic (i.e., admitting arbitrary insertions and deletions of points) 
$(2+\epsilon)$-approximation algorithm for the problem was presented in \cite{chanFully}.  This algorithm
features update times linear in $k$ and $1/\epsilon$, and polylogarithmic in the  \emph{aspect ratio} $\Delta = \maxdist/\mindist$ of the pointset, that is, 
the ratio between the maximum distance $\maxdist$ and minumum distance $\mindist$ of any two points in the set.  For the $k$-center
problem with $z$ outliers, a simple, fully
combinatorial 3-approximation algorithm running in
$O(k|W|^2\log(|W|))$ time was devised in \cite{Charikar2001}.  In \cite{MalkomesKCWM15,CeccarelloPP19} an
 adaptation of this latter algorithm was proposed for weighted
pointsets (where $z$ represents the aggregate weight of the
outliers). Recently, sequential $2$-approximation algorithms  were developed in
\cite{HarrisPST19,ChakrabartyGK20}, which, however, are based
on a more complex LP-based approach, less amenable to a practical
implementation. A bi-criteria randomized algorithm, only
suitable for (small) constant $k$, which returns a 2-approximation
as long as a slightly larger number of outliers is excluded from the objective function,
was presented in \cite{DingYW19}.

For the classical streaming setting, which seeks, at any step, a solution for
the \emph{full} stream seen so far, a $(4 + \epsilon)$- approximation algorithm
for the $k$-center problem with $z$ outliers was given
in \cite{McCutchenStreaming}. The approximation was later improved to
$3+\epsilon$ in \cite{CeccarelloPP19}. While the former algorithm requires
$O(kz/\epsilon)$ working memory space, the latter requires space
$O((k+z)(c/\epsilon)^D)$, where $c$ is a constant and 
$D$ is the \emph{doubling dimension} of the input stream
(a generalization of the notion of
Euclidean dimension to arbitrary metrics, formally defined in
Section~\ref{sec:doublingdim}).

\sloppy For the stricter sliding window setting, in \cite{Cohen} the
authors devised an algorithm able to compute a $(6 +
\epsilon)$-approximation to the $k$-center problem for the current
window, while keeping $O(k\log(\Delta)/\epsilon)$ points stored in
the working memory.  Based on the same techniques, the authors also
developed a $(3+\epsilon)$-approximation algorithm for the diameter of
the current window. In \cite{PellizzoniPP20}, we presented a
$(2+\epsilon)$-approximation algorithm for the $k$-center problem in
the sliding window setting, where the improved approximation is
obtained at the expense of a blow-up of a factor $O((c/\epsilon)^D)$
in the working memory, with $c$ constant.

For what concerns the $k$-center problem with $z$ outliers in the sliding
window setting, the only known algorithm was devised very recently in
\cite{DeBergMZ21}. At every time step, the algorithm maintains an
$\epsilon$-coreset for the problem on the
current window, namely, a subset of the window points, such that, if
used as input of any $c$-approximation sequential algorithm for
$k$-center with $z$ outliers, yields a $(c+\epsilon)$-approximate
solution to the problem on the entire window. The algorithm requires
the knowledge of the aspect ratio of the stream,  and uses
$\BO{\log (\Delta) kz(1/\epsilon)^D}$ working
memory. Moreover, since it relies on the execution of a sequential
algorithm for $k$-center with $z$ outliers on the coreset every time
a new point is added, it has an $\BO{\log (\Delta)
  \left(kz(1/\epsilon)^D\right)^3}$ update time per point, which makes
it very impractical for streams with high arrival rates, and values of
$k$ and $z$ not too small. In \cite{DeBergMZ21}, the authors also
prove that any sliding-window algorithm for $k$-center with $z$ outliers, which
features a $(1+\epsilon)$ approximation ratio, must use $\BOM{\log (\Delta) kz/\epsilon}$ working memory. 

In \cite{BravermanLLM16,BorassiELVZ20}, sliding window algorithms have
been proposed for the $k$-median and $k$-means clustering problems,
whose objective is to minimize the average distance and squared
distance of all window points from the closest centers, respectively.
For distributed solutions to the $k$-center problem (with and without
outliers) targeting volume rather than velocity of the data, see
\cite{CeccarelloPP19,BateniEFM21} and references therein.

The notion of $\alpha$-effective diameter was introduced in
\cite{PalmerGF02} in the context of graph analytics to characterize
the growth rate of the neighborhood function, but it naturally extends
to general metrics, providing a robust substitute of the diameter in
the presence of noise. For $\alpha \in (0,1)$, the $\alpha$-effective 
diameter of a metric dataset $W$ is the
minimum threshold such that the distances of at least $\alpha |W|^2$
pairs of window points fall below the threshold.

\subsection{Our Contribution} 
We present approximation algorithms for $k$-center with $z$ outliers
in the sliding window setting, which feature constant approximation
ratios and small working memory requirements and update times. As
customary for clustering in the big data realm, our algorithms hinge
on the maintenance of a \emph{coreset}, that is, a small subset of
representative points of the current window, from which an accurate
solution can be extracted. Crucial to the effectiveness of our
approach, is the introduction of \emph{weights} for coreset points,
where the weight of a point is (an estimate of) the
number of window points it represents. To efficiently maintain the
weights, we employ a succinct data structure inspired by the smooth
histograms of \cite{Braverman07}, which enable considerable space
savings if some slackness in the account of outliers is permitted.
 As an interesting
by-product, we also devise an algorithm that employs our coresets to approximate
the $\alpha$-effective diameter of the current window $W$.

The analysis of our algorithms is carried out as a function of two
design parameters $\delta$ and $\lambda$, which control, respectively,
the level of accuracy and the slackness in the account of outliers,
and as a function of a number of characteristics of the stream $S$,
namely, the values $\mindist$ and $\maxdist$, representing the minimum
and the maximum distance of two distinct points of $S$, and the
doubling dimension $D$ of $S$.

Our main results are listed below.

\begin{itemize}
\item 
A sliding window algorithm which, at any time, is able to return a set
of centers covering all but at most $z(1+\lambda)$ points of the
current window $W$, within a radius which is an $\BO{1}$ factor larger
than the optimal radius for $z$ ouliers. The algorithm requires a
working memory of size $\BO{\log (\maxdist/\mindist) (k+z)
  \log_{1+\lambda}(|W|)}$ and processes each point in time linear in
the working memory size.  By setting $\lambda = 1/(2z)$, the number of
uncovered points becomes at most $z$.
\item 
An improved algorithm with the same coverage guarantee as above, featuring a radius which is only a factor
$(3+O(\delta))$ larger than the optimal radius, at the expense
of an extra $O((c/\delta)^D)$ factor in both the working memory size and update time, for a
suitable constant $c$. 
\item
A sliding-window algorithm that, starting from a (possibly crude)
lower bound on the ratio between the $\alpha$-effective and the full
diameter of the window $W$, returns upper and lower upper bounds to
the $\alpha$-effective diameter of $W$. The algorithm features accuracy-space tradeoffs 
akin to those of the improved algorithm for $1$-center
with $z=0$ outliers.
\item
Experimental evidence that our algorithms feature good performance and provide accurate solutions. 
\end{itemize}

It is important to remark that our algorithms are \emph{fully
  oblivious} to the metric parameters $\mindist$, $\maxdist$, and $D$,
in the sense that the actual values of these parameters only influence
the analysis but are not needed for the algorithms to run.  This is a
very desirable feature, since, in practice, these values are difficult
to estimate.

Compared to the algorithm of \cite{DeBergMZ21} for $k$-center with $z$
outliers, our algorithms feature a considerably lower update time,
which makes them  practically viable.  Moreover, in the case of noisy streams for
which $z = \BOM{\log |W|}$, our algorithms require considerably less
working memory, as long as some slackness in the number of outliers can be 
tolerated. Finally, while the algorithm of \cite{DeBergMZ21} requires
the knowledge of $\mindist$ and $\maxdist$, our algorithms are oblivious to these
values. 

\subsection{Organization of the Paper}
The rest of the paper is structured as follows.
Section~\ref{sec:preliminaries} provides 
preliminary definitions. Sections~\ref{sec:algorithm} and
\ref{sec:effdiameter} present, respectively, the
algorithms for the k-center problem with $z$ outliers and for the
effective diameter. Section~\ref{sec:experiments} reports on the
experimental results. Section~\ref{sec:conclusions} concludes the paper
with some final remarks and pointers to relevant open problems.

\section{Preliminaries} \label{sec:preliminaries}
Consider a (possibly unbounded) stream $S$ of points from some metric
space with distance function $\distn(\cdot, \cdot)$. At any time $t$, let $W$ denote
the set of the last $N=|W|$ points arrived, for a fixed window length $N$. 
In the sliding window model, for a given computational problem, we aim at developing algorithms which, at any time $t$, are able to solve the instance represented
by the current window $W$, using working memory considerably smaller than $N$ (possibly, constant or logarithmic in $N$).

\subsection{Definition of the Problems} \label{sec:definitions}
For any  point $p \in W$ and any subset  $C \subseteq W$, 
we use the notation $\distn(p,C) = \min_{q \in C} \distn(p,q)$,
and define the \emph{radius of $C$ with respect to $W$} as
\[
r_C(W) = \max_{p \in W} \distn(p,C).
\] 
For a positive integer $k<|W|$, the \emph{$k$-center problem} requires
to find a subset $C \subseteq W$ of $k$ centers which minimizes
$r_C(W)$.  For a given $W$ and $k$, we denote the radius of the
optimal solution of this problem by $r_{k}^*(W)$. Given any radius
value $r$, a subset $C \subseteq W$ with $r_C(W) \leq 2r$, can be
incrementally built using the \emph{greedy strategy} of
\cite{HochbaumS85}: starting from an arbitrary center, a new center,
selected among the points of $W$ at distance $>2r$ from the current
centers, is iteratively added to $C$ until all points of $W$ are at
distance at most $2r$ from $C$. An easy argument shows that if $r \geq
r_{k}^*(W)$, the set $C$ obtained in this fashion has size at most
$k$. By combining this strategy with a suitable guessing protocol, a
2-approximate solution to the $k$-center problem for $W$ is obtained.

Note that any subset $C \subseteq W$
induces a partition of $W$ into $|C|$ clusters, by assigning each
point to its closest center (with ties broken arbitrarily).

In this paper, we focus on the following important extension to the $k$-center problem. For positive $k,z < |W|$, the
\emph{$k$-center problem with $z$ outliers} requires to find a subset $C
\subseteq W$ of size $k$ minimizing $r_C(W-Z_C)$, where $Z_C$ is
the set of $z$ points in $W$ with the largest distances from $C$, which are
regarded as outliers to be discarded from the clustering. We denote the 
radius of the optimal
solution of this problem by $r_{k, z}^*(W)$.  Observe that the
$k$-center problem with $z$ outliers reduces to the $k$-center problem for
$z=0$. Also, it is straightforward to argue that the optimal solution
of the $k$-center problem (without outliers) with $k+z$ centers has a
radius not larger than the optimal solution of the problem
with $k$ centers and $z$ outliers, that is,
\begin{equation}
r_{k+z}^*(W) \le r_{k,z}^*(W).
\label{eq:radius-relation}
\end{equation}

In a more general formulation of the $k$-center problem with $z$ outliers,
each point $p \in W$ carries a positive integer weight $w(p)$, and the 
desired set $C$ of $k$ centers must minimize $r_C(W-Z_C)$, where $Z_C$ is
the set of points with the largest distances from $C$, of maximum cardinality
and aggregate weight at most $z$. We will refer to this weighted formulation as
\emph{weighted $k$-center with $z$ outliers}.

The algorithms presented in this paper for $k$-center with
$z$ outliers crucially rely on the extraction of a succinct
\emph{coreset} $T$ from the (possibly large) input $W$, so that a
solution to the problem can be efficiently computed by running  a 
sequential algorithm on $T$ rather than on $W$. The quality of a
coreset $T$ is captured by the following definition.

\begin{Definition} \label{def:coreset}
Given a pointset $W$ and a value $\epsilon >0$, a subset $T \subseteq
W$ is an \emph{$\epsilon$-coreset for $W$} w.r.t. the $k$-center
problem with $z$ outliers 
if $\max_{p \in W} \distn(p, T) \leq \epsilon r^*_{k,z}(W)$.
\end{Definition}

An $\epsilon$-coreset $T$ of $W$ ensures
that each point in $W$ is ``represented'' by a close enough point in
$T$, where closeness is defined w.r.t.\ $\epsilon$ and
$r^*_{k,z}(W)$.  In fact, our algorithms will make use of
\emph{weighted coresets}, where, additionally, each coreset point $p
\in T$ features a weight which is (an approximation of) the number of
points of $W$ represented by $p$.

An important characteristic of a pointset $W$ is its \emph{diameter}, defined as $\Delta_W =
\max_{p,q \in W} \distn(p,q)$, which can be computed exactly in
quadratic time. 
The diameter is very sensitive to noise in the dataset and, in the
presence of outliers, its value might turn out to be scarcely
representative of most pairwise distances in $W$. Thus, the more robust
notion of \emph{effective diameter} has been introduced in
\cite{PalmerGF02}. Let $d_{1,W}, d_{2,W}, \ldots$ be an
enumeration of the $|W|^2$ distances between all pairs of points of
$W$, in non-decreasing order. For a given parameter $\alpha \in
(0,1)$, the \emph{$\alpha$-effective diameter} of $W$ is defined as
$\Delta^{\alpha}_W = d_{\lceil \alpha |W|^2\rceil,W}$, namely,
the smallest value such that at least $\alpha |W|^2$ pairs of points in $W$ are
within distance $\Delta^{\alpha}_W$.

\subsection{Doubling Dimension} \label{sec:doublingdim}
The analysis of our algorithms will be carried out as a function of
a number of relevant parameters, including  the dimensionality of the data. To deal with arbitrary metric spaces, we resort to the
following, well-established general notion of dimensionality. For any $x \in W$ and $r
>0$, the \emph{ball of radius $r$ centered at $x$}, denoted as
$B(x,r)$, is the subset of all points of $W$ at distance at most $r$
from $x$. The \emph{doubling dimension} of $W$ is the minimum value
$D$ such that, for all $x \in W$, any ball $B(x,r)$ is contained in
the union of at most $2^D$ balls of radius $r/2$ centered at points of
$W$. The notion of doubling dimension has been used extensively in previous
works (see \cite{GottliebKK14,PellizzoniPP20} and references therein).

\section{$k$-Center with $z$ Outliers} \label{sec:algorithm}
Let $S$ be a (possibly unbounded) stream of points from some metric
space, and let $N$ be the selected window length. For any point $p \in
S$, its \emph{Time-To-Live} $\TTL(p)$ is $N$ when $p$ arrives, and it
decreases by 1 at each subsequent step. We say that $p$ is
$\emph{active}$ when $\TTL(p) >0$, and that it $\emph{expires}$ when
$\TTL(p)$ becomes 0. For convenience, the analysis will also consider
expired points with negative TTL's. At any time  $t$, the current window $W$ 
consists of all arrived points with positive TTL, hence $|W|=N$.

In this section, we present coreset-based algorithms which, at any
time $t$, are able to return accurate approximate solutions to the
$k$-center problem with $z$ outliers for $W$.  The section is
structured as follows.  Subsection~\ref{sec:coreset} describes and
analyzes the weighted coreset construction. Subsection~\ref{sec:solution} discusses how to extract the
final solution from the weighted coreset whose radius is at most a
constant factor away from $r_{k,z}^*(W)$, as long as a slightly larger
number of outliers is tolerated. Subsection~\ref{sec:obliviousness} shows how to remove
an assumption made to simplify the presentation. 
Finally,
Subsection~\ref{sec:ddimension} shows that for spaces of bounded
doubling dimension, the approximation factor can be lowered to a mere
$3 + \epsilon$, for any fixed $\epsilon$, at the expense of larger
working memory requirements.

\subsection{Weighted Coreset Construction} \label{sec:coreset}

\subsubsection{Algorithm} The proposed coreset construction hinges upon the
approach by \cite{Cohen} for $k$-center without
outliers, with major extensions introduced to maintain
weights. Let $\mindist$ and $\maxdist$ denote, respectively, the minimum and
maximum distances between any two distinct points of the stream.
For a user-defined constant $\beta \in (0,1]$, let
\[
\Gamma = 
\{(1+\beta)^i : 
\lfloor \log_{1+\beta} \mindist \rfloor 
\leq i \leq
\lceil \log_{1+\beta} \maxdist \rceil \},
\]
The values in $\Gamma$ will be used as guesses of the optimal radius
$r^*_{k+z}(W)$ of a $(k+z)$-center clustering without outliers of the
current window (recall that $r^*_{k+z}(W)$ is a lower bound to the
optimal radius $r^*_{k,z}(W)$ for the problem with $z$ outliers), and
the algorithm will maintain suitable data structures capable of
identifying the right guess.  For ease of presentation, we assume for
now that $\mindist$ and $\maxdist$ are known to the algorithm.  In
Subsection \ref{sec:obliviousness}, we will show how the assumption can
be removed by maintaining estimates of the two values. 
For each guess $\gamma$, the algorithm maintains three sets of active
points, namely
$A_{\gamma}$, $R_{\gamma}$ and $O_{\gamma}$, and the coreset is
extracted from these sets.  $A_{\gamma}$ is a small set of active
points, called \emph{attraction points}, such that, for any
two distinct $a_1,a_2 \in A_{\gamma}$, $\distn(a_1, a_2) > 2\gamma$.

At every time $t$, the arrival of a new point $p$ is
handled as follows, for every guess $\gamma$.
If there exist attraction points $a$ such that $\distn(p, a) \leq
2\gamma$, we define $a_{\gamma}(p)$ as the one with minimum TTL, and
say that \emph{$p$ is attracted by $a_{\gamma}(p)$}. Otherwise, $p$
becomes a new attraction point in $A_{\gamma}$, and we let
$a_{\gamma}(p) = p$ (i.e., $p$ is attracted by itself). Set
$R_{\gamma}$ maintains, for each $a \in A_{\gamma}$,
one \emph{representative} $r_{\gamma}(a)$, defined as the most recent
point attracted by $a$. Note that while $a_{\gamma}(p)$ is fixed at
$p$'s arrival, the representative $r_{\gamma}(a)$ may change with
time. When an attraction point $a$ expires, its representative
$r_{\gamma}(a)$ becomes an \emph{orphan} and is moved to the set
$O_{\gamma}$. 

Since the pairwise distance between points of $A_{\gamma}$ is
$>2\gamma$, if $|A_{\gamma}| \geq k+z+1$, we clearly have that $\gamma
< r^*_{k+z}(W)$, and, as it will be seen below, the points in
$A_{\gamma} \cup R_{\gamma} \cup O_{\gamma}$ will not be used to
extract the coreset. Therefore, to save memory, we set $k+z+1$ as a
threshold for $|A_{\gamma}|$: when $|A_{\gamma}| = k+z+1$ and the
newly arrived point qualifies to be an attraction point, the algorithm
discards the point $a \in A_{\gamma}$ with minimum TTL, and moves its
representative $r_{\gamma}(a)$ to $O_{\gamma}$. As a further space
saving, all points in $O_{\gamma}$ older than $a$ are discarded,
since throughout their residual lifespan $|A_{\gamma}| \geq k+z+1$,
hence they cannot contribute to a valid coreset. 

At any time $t$, the coreset for the $k$-center problem with $z$
outliers, w.r.t.\ the current window $W$, is obtained as $T =
R_{\hat{\gamma}} \cup O_{\hat{\gamma}}$, where $\hat{\gamma}$ is the
smallest guess such that: (i) $|A_{\hat{\gamma}}| \leq k+z$; and (ii)
by running the simple greedy strategy of \cite{HochbaumS85}, reviewed
in Section~\ref{sec:definitions}, a set $C$
of $k+z$ points can be selected from $A_{\hat{\gamma}} \cup
R_{\hat{\gamma}}
\cup O_{\hat{\gamma}}$, such that any other point in this set is at
distance at most $2\hat{\gamma}$ from a selected point.
To ensure that an accurate solution to the $k$-center problem with $z$
outliers can be extracted from the coreset $T$, we need
to weigh each point $p \in T$ with (a suitable accurate estimate of)
the number of window points for which $p$ can act as
a \emph{proxy}. This requires to maintain additional information with
the points of the various sets $R_{\gamma}$ and $O_{\gamma}$, as
explained below.

For each guess $\gamma$ and each active point $p \in W$, we define its
proxy $\pi_{\gamma}(p) \in R_{\gamma} \cup O_{\gamma}$ as the most
recent active point $r$ such that both $p$ and $r$ are attracted by
$a_{\gamma}(p)$.  (Note that $r$ may be an orphan if $a_{\gamma}(p)$
was discarded from $A_\gamma$.) Thus, $\pi_{\gamma}(p) =
r_{\gamma}(a_{\gamma}(p))$. Therefore, at any time $t$, the proxy
function $\pi_{\gamma}(\cdot)$ defines a mapping between active points
and points of $O_{\gamma} \cup R_{\gamma}$, and, for every $r \in
O_{\gamma} \cup R_{\gamma}$ we define its \emph{weight} $w_{\gamma}(r)
= |\{p \in W \; : \; \pi_{\gamma}(p)=r\}|$.  For each $r \in
R_{\gamma} \cup O_{\gamma}$, our algorithm maintains a
\emph{histogram} $L_{r,\gamma} = \{ (t_{r,1}, c_{r,1}), (t_{r,2},
c_{r,2}), \ldots \}$, which is a list of pairs (timestamp, weight)
such that there are $c_{r,j}$ points $p$ assigned to $r$ (i.e., for
which $\pi_{\gamma}(p)=r$) that arrived at or after time
$t_{r,j}$.  When a point $p$ arrives at time $t$, the histograms are
updated as follows\footnote{For the sake of readability, from now on we drop
the subscript $\gamma$ from histograms and weights, when clear from the
context.}.  If $p$ becomes a new attraction point (hence,
$p=a_{\gamma}(p)=r_{\gamma}(p)$), a new histogram $L_p=\{ (t, 1) \}$
is created and is assigned to $p$. If instead $p$ is attracted by some
$a \in A_\gamma$, then $p$ becomes the representative $r_\gamma(a)$
and inherits the histogram from the previous
representative, modified by increasing all weights by 1 and adding the
new pair $(t, 1)$. Also, all histogram entries with timestamp 
$t-|W|$ are discarded, since they refer to the point which expired at time $t$. 
The pairs in each histogram are naturally sorted by
increasing order of timestamps and decreasing order of weight. Observe
that when a representative $r$ becomes an orphan, hence it is moved
from $R_\gamma$ to $O_\gamma$, its histogram $L_r$ does not acquire new entries
until $r$ expires.

At any time $t$, for a histogram $L_r$, we denote by $c_{L_r}$ the
 weight of the pair in $L_r$
with the smallest timestamp (which is greater than or equal to $t-|W|+1$ by virtue of the elimination of the old entries described above)
 It is easy to see that
$c_{L_r} = w(r)$, that is, the number of points $p$ for which $r$ is
the proxy. Unfortunately, keeping the full histogram $L_r$ for each
$r \in R_\gamma \cup O_\gamma$ requires a working memory of size $\BT{|W|}$, which
is
far beyond the space bound targeted by sliding window algorithms.
Therefore, taking inspiration from the smooth histograms
of \cite{Braverman07}, we maintain in $L_r$ only a trimmed version of
the full list, which however ensures that $c_{L_r}$ is an estimate of
$w(r)$ with a controlled level of accuracy. Specifically, let $\lambda
>0$ be a user-defined accuracy parameter. Every time a histogram $L_r$
is updated, a scan of the pairs is performed which implements the
following trimming:
\begin{itemize}
\item
The first pair $(t,c)$ is kept in the histogram.
\item
If a pair $(t,c)$ is kept in the histogram, 
all subsequent pairs $(t',c')$ with $t' > t$ and 
$c \leq (1+\lambda)c'$ are deleted, except for the
last such pair, if any.
\end{itemize}

At any time $t$, the weighted coreset that will be used to solve the
$k$-center problem with $z$ outliers for the current window $W$
consists of the set $T = R_{\hat{\gamma}} \cup O_{\hat{\gamma}}$,
where the guess $\hat{\gamma}$ is computed as described above, and
each $r \in T$ is assigned weight $\tilde{w}(r) = c_{L_r}$.  As it
will be shown in the next subsection, the $\tilde{w}(r)$'s are good
approximations of the true weights $w(r)$'s, and this will provide
good bi-criteria approximation quality for the returned solution.

The pseudocode detailing the algorithm is provided below. 
The arrival of a new point $p$ at time $t$ is handled
by the main Procedure~$\proc{update}(p,t)$ (Algorithm~\ref{alg:update}) 
which, for every guess
$\gamma$, invokes in turn
Procedure~$\proc{insertAttraction}(p,\gamma)$ (Algorithm~\ref{alg:insert}), 
when $p$ must be added
to $A_{\gamma}$, or
Procedure~$\proc{updateHistrograms}(L_{p,\gamma})$ (Algorithm~\ref{alg:histograms}), when $p$ becomes a
new representative of some existing point of $A_{\gamma}$. To extract the weighted coreset, Procedure~$\proc{extractCoreset}()$ (Algorithm~\ref{alg:coreset})
is executed.

\begin{algorithm}[h] 
\small
    \SetAlgoLined
    \ForEach{$\gamma \in \Gamma$}{
        \ForEach{$\mbox{expired } p \in A_{\gamma}$}{
            $A_{\gamma} \gets A_{\gamma} \setminus \{p\}$ \\
            $\mbox{Move } r_{\gamma}(p) \mbox{ from } R_{\gamma} \mbox{ to } O_{\gamma} $ \\
        }
   \ForEach{$r \in O_{\gamma}$}{
   \uIf{$r$ \emph{is} expired}{Remove $r$ (and its histogram) from $O_{\gamma}$} 
   Remove from $L_{r,\gamma}$ the entry (if any) with timestamp $t-|W|$
}
        $x \gets \argmin_{ q \in A_{\gamma} : \distn(p, q) \leq 2\gamma } \TTL(q)$ \\
        \uIf{$x == $ null}{
            $\proc{insertAttraction}(p, \gamma)$ \\
            $L_{p,\gamma} \gets \{ ( t, 1 ) \}$ \\ \label{creation}
        }
        \Else{
            Move the content of $L_{r_{\gamma}(x),\gamma}$ to $L_{p,\gamma}$ \\
            $\proc{updateHistogram}(L_{p,\gamma})$ \\
            Set $r_{\gamma}(x) = p$ in $R_{\gamma}$\\
        } 
       
    }

\caption{$\proc{update}(p, t)$} \label{alg:update}
\end{algorithm}

\begin{algorithm}[h] 
\small
    \SetAlgoLined
    $A_{\gamma} \gets A_{\gamma} \cup \{ p \}$ \\
    $r_{\gamma}(p) \gets p$ \\
    $R_{\gamma} \gets R_{\gamma} \cup \{r_{\gamma}(p)\}$ \\
    \If{$|A_{\gamma}| > k+z+1$}{
        $v_{old} \gets \argmin_{v \in A_{\gamma}} \TTL(v)$ \\
        $A_{\gamma} \gets A_{\gamma} \setminus \{v_{old}\}$ \\
        Move $r_{\gamma}(v_{old})$ from $R_{\gamma}$ to $O_{\gamma}$  \\
    }
    \If{ $|A_{\gamma}| > k+z$}{
        $t_{\rm min} \gets \min_{v \in A_{\gamma}} \TTL(v)$ \\
        Remove from $O_{\gamma}$ all $q$ with $\TTL(q) < t_{\rm min}$ (and their histograms)\\
    }   
\caption{$\proc{insertAttraction}(p, \gamma)$} \label{alg:insert}
\end{algorithm}


\begin{algorithm}[h] 
\small
\DontPrintSemicolon
    \SetAlgoLined
Let $L[i]=(L[i].t,L[i].c)$ 
denote the $i$th pair in $L$, for $i=1,2,\ldots |L|$ \\
\For{$i = 1$ to $|L|$}{$L[i].c \gets L[i].c+1$} 
    Append $(t,1)$ to $L$ \\ \label{append}
    Create a new histogram $M = \{ L[1] \}$ \\
    $last = 1$ \\
    \For{$i = 2$ to $|L|-1$}{
        \If{$L[last].c > (1+\lambda)L[i+1].c$}{
            Append $L[i]$ to $M$ \\
            $last = i$ \\
        }
    }
    Append $L[|L|]$ to $M$ \\
    $L=M$ \\
\caption{$\proc{updateHistogram}(L)$} \label{alg:histograms}
\end{algorithm}

\begin{algorithm}[h] 
\small
\DontPrintSemicolon
\SetAlgoLined
\For{increasing $\gamma \in \Gamma$ such that $|A_{\gamma}| \leq
 k+z$}{$C \gets \emptyset$\\
\For{$p \in A_{\gamma} \cup O_{\gamma} \cup R_{\gamma}$}{
\lIf{$\distn(p, C) > 2\gamma$}{$C \gets C \cup \{ p \}$}}
\If{$|C| \leq k+z$}{$\hat{\gamma} \gets \gamma$ \\ \textbf{break};}
}
$T \gets R_{\hat{\gamma}} \cup O_{\hat{\gamma}}$ \\
\lForEach{$r \in T$}{$\tilde{w}(r) = c_{L_r}$}
\Return $T$ together with the approximate weights.
\caption{$\proc{extractCoreset}()$} \label{alg:coreset}
\end{algorithm}

\subsubsection{Analysis} 
The following two technical lemmas
state important properties of the sets $A_{\gamma}$,
$O_{\gamma}$ and $R_{\gamma}$ and of the histograms maintained by the algorithm. 
\begin{Lemma} \label{lem:technical}
At any time $t$, the following properties hold for every $\gamma \in \Gamma$:
\begin{enumerate}
\item
If $|A_{\gamma}| \leq k+z$, then 
$\max_{q \in W} \distn(q, \pi_{\gamma}(q)) \leq 4\gamma$.
\item
$|A_{\gamma}|, |R_{\gamma}|, |O_{\gamma}| \leq k+z+1$.
\end{enumerate}
\end{Lemma}
\begin{proof}
In order to prove the lemma, it suffices to show  that if the properties hold after the processing of the  $(t-1)$-th point, they are inductively maintained after the invocation of $\proc{update}(p, t)$.
The proof makes use of essentially the same arguments employed
in \cite[Lemmas 7, 8]{Cohen}, straightforwardly adapted
to account for the fact that, in our algorithm, each new point which does not
become an attraction point is made representative of a single attraction 
point, whereas in \cite{Cohen} it would be made representative of all attraction points 
at distance at most $2\gamma$.
\end{proof}

Recall that at any time $t$ and for any guess $\gamma$, the  histogram $L_r$ associated with each point $r \in R_{\gamma} \cup O_{\gamma}$ is a list of pairs $(t_{r,i},c_{r,i})$ indicating that there are currently $c_{r,i}$ points arrived after time $\geq t_{r,i}$, whose proxy is $r$, for $i=1,2, \ldots$. We have:

\begin{Lemma} \label{lem:cleanup}
For any time $t$, guess $\gamma$, and $r \in R_{\gamma} \cup O_{\gamma}$, the following
properties hold for $L_r$. 
\begin{enumerate}
\item 
For every $1 \leq i \leq |L_r|$,  $c_{r,i} \leq |W|$.
\label{prop1}
\item
For every $1 \leq i \leq |L_r|-1$, 
$c_{r,i} \leq (1+\lambda)c_{r,i+1}$ or 
$c_{r,i} = 1+c_{r,i+1} > (1+\lambda)c_{r,i+1}$.
\label{prop2}
\item
For every $1 \leq i \leq |L_r|-2$, 
$c_{r,i} > (1+\lambda)c_{r,i+2}$.
\label{prop3}
\item
$|L_r| \in \BO{ \log_{1+\lambda} |W|}$.
\label{prop4}
\end{enumerate}
\end{Lemma}
\begin{proof}
First, observe that Property 4 is an immediate consequence of 
Properties 1 and 3. Hence, we are left with proving 
Properties 1, 2 and 3. The properties clearly hold when the histogram $L_r$ is first created (Line~\ref{creation} of $\proc{update}$), thus we only need to show that
if they hold prior to an invocation of $\proc{updateHistogram($L_r$)}$, they continue to hold
for the histogram $M$ created by $\proc{updateHistogram($L_r$)}$,
which becomes the new $L_r$ at the end of the procedure. 
Property~1 holds since weights are updated only as long as
$a_\gamma(r)$, the oldest point accounted for in the histogram, is active, hence
weights always represent sizes of subsets of points of the same window, and thus they never exceed $|W|$. Consider now Properties~2 and 3. 
Let $L'_r$ denote the histogram $L_r$ after the execution of 
Line~\ref{append} of
$\proc{updateHistogram($L_r$)}$ (i.e., after the increment of the
weights and the addition of $(t,1)$ at the end of the list).
It is easy to argue  that Properties~2 and~3
continue to hold for $L'_r$. Let us now show that they  also hold 
for histogram $M$ at the end of the procedure.
As for Property~2,
consider two adjacent pairs $(t_{r,i}, c_{r,i})$ and $(t_{r,i+1},
c_{r,i+1})$ in $M$, with $c_{r,i} > (1+\lambda)c_{r,i+1}$.  Then, two
pairs with the same timestamps and weights must exist in $L'_r$,
and we can argue that these two
pairs must be adjacent in $L'_r$, hence their weights must differ 
by 1, since Property~2 holds for $L'_r$. Indeed, if the
two pairs were not adjacent in $L'_r$, then at least one pair with timestamp
between $t_{r,i}$ and $t_{r,i+1}$ must have been removed in the for loop of
$\proc{updateHistogram($L_r$)}$. However, this not possible because, by the
way the loop operates, this would ensure that $c_{r,i} \leq
(1+\lambda)c_{r,i+1}$.  Finally, Property~3 is enforced by
the for loop of $\proc{updateHistogram($L_r$)}$.
\end{proof}

The following theorem states the main properties of the weighted
coreset $T$ computed by our algorithm.

\begin{Theorem} \label{thm:coreset}
At any time $t$, the weighted coreset $T$ returned by 
$\proc{extractCoreset}()$
is a $4(1 + \beta)$-coreset of size $O(k+z)$
for the current window $W$ w.r.t. the $k$-center
problem  with $z$ outliers. Moreover,
for each point $r \in T$, we have
$w(r)/(1+\lambda) \leq \tilde{w}(r) \leq w(r)$.
\end{Theorem}
\begin{proof}
Recall that $T = R_{\hat{\gamma}} \cup O_{\hat{\gamma}}$, where
$\hat{\gamma}$ is the minimum guess for which the following two conditions are
verified:
$|A_{\hat{\gamma}}| \leq k+z$, and the greedy selection
strategy of \cite{HochbaumS85} applied to $A_{\hat{\gamma}} \cup R_{\hat{\gamma}}
\cup O_{\hat{\gamma}}$ returns a set $C$ of 
$k+z$ points such that $\distn(p,C) \leq 2\hat{\gamma}$,
for every $p \in A_{\hat{\gamma}} \cup R_{\hat{\gamma}}
\cup O_{\hat{\gamma}}$. It is easy to see that 
these two conditions are surely verified for 
any $\gamma \geq r_{k+z}^*(W)$. Therefore, considering
the density of the guesses in $\Gamma$, 
we must have $\hat{\gamma} \leq (1+\beta)r_{k+z}^*(W)$. 
By
Lemma~\ref{lem:technical}, we have that  $|T| = O(k+z)$ and
$\max_{p \in W} \distn(p, T)
\leq 4 \hat{\gamma}$, which implies $\max_{p \in W} \distn(p, T)
\leq 4 (1 + \beta) r_{k+z}^*(W) \leq 4
(1 + \beta) r_{k,z}^*(W)$.

We now show the relation concerning the approximate weights
$\tilde{w}(r)$.  Recall that for each $r \in T$, the value
$\tilde{w}(r)$ is set equal to $c_{L_r}$, which is the weight
component of the pair in $L_r$ with smallest timestamp $\geq t-|W|+1$.
Let $(t_{r,i},c_{r,i})$ be the pair of $L_r$ such that
$c_{r,i}=c_{L_r}$. If $i>1$, then the relation $c_{r,i} \leq w(r) \leq
c_{r,i-1}$ must hold, and Property~\ref{prop2} of
Lemma~\ref{lem:cleanup} ensures that $c_{r,i-1} \leq c_{r,i}
(1+\lambda)$ or $c_{r,i-1} = c_{r,i}+1$. In the former case, we have
$w(r)/(1+\lambda) \leq c_{r,i}=c_{L_r} \leq w(r)$, while in the latter
case it is easy to see that $c_{L_r} = w(r)$. If instead $i=1$ it is
easy to see that $c_{L_r} = w(r)$, since the first pair of the
histogram is always relative to the arrival of the attraction point
$a_{\gamma}(r)$, hence the weight of such pair accounts for all points assigned
to it.
\end{proof}

The following theorem analyzes the space and time performance of
our coreset construction strategy. 
\begin{Theorem} \label{thm:performance}
The data structures used by our coreset construction strategy
require a working memory of size 
$\BO{\log_{1+\beta}
  (\maxdist/\mindist) \cdot (k+z) \cdot \log_{1+\lambda}(|W|)}$. Moreover,
Procedure $\proc{update}$ can be implemented to run in time
\[
\BO{\log_{1+\beta}(\maxdist/\mindist) \cdot \left( (k+z) +  \log_{1+\lambda}(|W|) \right)},
\]
while Procedure $\proc{extractCoreset}$ can be implemented to run
in time 
\[
\BO{\log_{1+\beta}(\maxdist/\mindist) \cdot (k+z)^2}.
\]
If binary search is used to find the value $\hat{\gamma}$,
the running time of  $\proc{extractCoreset}$  decreases to
\[
\BO{ \log(\log_{1+\beta}(\maxdist/\mindist)) \cdot (k+z)^2}.
\]
\end{Theorem}
\begin{proof}
The bound on the working memory is an immediate consequence of 
Lemmas~\ref{lem:technical}  and \ref{lem:cleanup}, and of
the fact that $|\Gamma| = \BO{\log_{1+\beta}(\maxdist/\mindist)}$.
Procedure $\proc{update}$ is easily implemented through a constant number of 
linear scans of $A_{\gamma}$, $R_{\gamma}$, and $O_{\gamma}$, for every $\gamma \in \Gamma$,
and a linear scan of at most one histogram. 
For what concerns Procedure
$\proc{extractCoreset}$, for each $\gamma \in \Gamma$, 
the computation of $C$ requires time quadratic in
$|A_{\gamma}|+|R_{\gamma}|+|O_{\gamma}|$, and the number of
guesses $\gamma \in \Gamma$ to be checked are
at most $|\Gamma|$, if a linear search is used,
and $\BO{\log (|\Gamma|)}$, if binary search is used. 
Also, time $\BO{|T|}$ is needed to 
compute $\tilde{w}(r)$ for all $r\in T$.
Based on these
considerations, the complexity bounds
follow again from Lemmas~\ref{lem:technical}  and \ref{lem:cleanup}
and the size of $|\Gamma|$.
\end{proof}

\subsection{Computation of the Solution from the Coreset} \label{sec:solution}
At any time $t$, to compute a solution to the $k$-center problem with $z$ outliers on the window $W$,
we first determine a weighted coreset $T$, as explained in the previous subsection, and
then, on $T$, we run a sequential strategy for the weighted variant of the problem. 
To this purpose, we make use of the
algorithm developed in \cite{MalkomesKCWM15,CeccarelloPP19}, which
generalizes the sequential algorithm of \cite{Charikar2001} to the
weighted case. Given in input a weighted set $T$, a value $k$, a
precision parameter $\epsilon$ and a radius $\rho$, the algorithm
computes a set $X$ of $k$ centers incrementally in at most $k$
iterations.  Initially, all points are considered \emph{uncovered}.
At each iteration, the next center is selected as the point which
maximizes the aggregate weight of the yet uncovered points within
distance $(1 + 2\epsilon) \rho$, and such a center \emph{covers} all points
in $T$
within the larger distance $(3 + 4\epsilon) \rho$. The algorithm
terminates when either $|X| = k$ or no uncovered points remain,
returning $X$ and the subset $T'$ of uncovered points.
This algorithm is implemented by procedure
$\proc{outliersCluster}(T, k, \rho, \epsilon)$. 
The next lemma
states an important property related to the use of  $\proc{outliersCluster}$
in our context.

\begin{algorithm}[h] 
\small
    \SetAlgoLined
    $T' = T$ \\
    $X = \emptyset$ \\
    \While{$|X| < k$ and $T' \neq \emptyset$}{
        \For{$r \in T$}{
            $B_r = \{ v : \ v \in T' \mbox{\it and } \distn(v, r) \leq (1+2\epsilon) \rho \}$ \\
        }
        $x = \argmax_{r \in T} \sum_{v \in B_r} \tilde{w}(v)$ \\
        $X = X \cup \{ x \}$ \\
        $E_x = \{ v : \ v \in T' \mbox{\it and } \distn(v, x) \leq (3+4\epsilon)\rho \}$ \\
        $T' = T' \setminus E_x$ \\
    }
\Return $X, T'$
\caption{\proc{outliersCluster}($T$, $k$, $\rho$, $\epsilon$)}
\label{alg:outlierscluster}
\end{algorithm}

\begin{Lemma} \label{lem:outclust_apprx}
For any time $t$, let $T=R_{\hat{\gamma}} \cup O_{\hat{\gamma}}$ be
the weighted coreset returned by $\proc{extractCoreset}()$.
For any $\rho \geq
r^*_{k,z}(W)$,  the invocation $\proc{outliersCluster}(T,
k, \rho, \epsilon)$ with $\epsilon=4(1+\beta)$ returns a set of 
centers $X$ and a set of uncovered points $T' \subseteq T$, such that
\begin{itemize}
\item
$\distn(r, X) \leq (3 + 4\epsilon) \rho \quad \forall r \in T \setminus T'$
\item
$\sum_{r \in T'} \tilde{w}(r)  \leq z$.
\end{itemize}
\end{Lemma}
\begin{proof}
The proof can be obtained as a simple technical 
adaptation of the one of \cite[Lemma 5]{CeccarelloPP19}
to the case of approximate
weights. For clarity, we detail the entire proof rather than highlighting
the changes only. The bound on $\distn(r, X)$
for every $r \in T \setminus T'$ is directly enforced by the algorithm. 
We are left to show that $\sum_{r \in T'} \tilde{w}(r)  \leq z$.
If $|X|<k$, then $T' = \emptyset$ and the claim holds vacuously. 
We now concentrate on the case $|X|=k$. 
Consider the $i$-th
iteration of the while loop of $\proc{outliersCluster}(T,k,\rho,\epsilon)$
and define $x_i$ as the center of $X$ selected in the iteration, and 
$T'_i$ as the set $T'$ of uncovered points at the beginning
of the iteration. Recall that $x_i$ is the
point of $T$  which maximizes the cumulative approximate weight of
the set $B_{x_i}$ of uncovered points in $T'_i$  at distance at most
$(1+2\epsilon)\cdot \rho$ from $x_i$, and that
the set $E_{x_i}$ of all uncovered points at distance 
at most $(3+4\epsilon)\cdot \rho$ from $x_i$ is removed from
$T'_i$  at the end of the iteration. 
Let 
\[
\sigma_T = \sum_{r \in T} \tilde{w}(r).
\]
We now show that 
\begin{equation} \label{eq:weight-outliers}
\sum_{i=1}^k \sum_{r \in E_{x_i}} \tilde{w}(r) \ge \sigma_T-z,
\end{equation}
which will immediately imply that $\sum_{r \in T'} \tilde{w}(r) \leq
z$.  To this purpose, let $O$ be an optimal set of $k$ centers for the
current window $W$, and let $Z$ be the set of at most $z$ outliers at
distance greater than $r^*_{k,z}(W)$ from $O$.  Let also $\tilde{W}$
be a subset of the current window which, for every $r \in T$, contains
exactly $\tilde{w}(r)$ points $q$,
including $r$, for which $r$ is a proxy, that is, points $q \in W$ such that $\pi_{\hat{\gamma}} (q)=r$. 
Note that $\tilde{W}$ is well defined since, by Theorem~\ref{thm:coreset}, $\tilde{w}(r)$ is always 
less than or equal to the actual weight of $r$. 
 For each $o \in
O$, define $C_o \subseteq \tilde{W} \setminus Z$ as the set of
nonoutlier points in $\tilde{W}$ which are closer to $o$ than to any
other center of $O$, with ties broken arbitrarily. It is important to
remark, that while there may be some optimal center $o \in O$ which is
not in $\tilde{W}$, its proxy $\pi_{\hat{\gamma}} (o)$ is in $T$,
hence it is guaranteed to be in $\tilde{W}$. To
prove Eq.~(\ref{eq:weight-outliers}), it is sufficient to exhibit an
ordering $o_1, o_2, \ldots, o_k$ of the centers in $O$ so that, for
every $1 \leq i \leq k$, it holds
\[
\sum_{j=1}^i \sum_{r \in E_{x_j}} \tilde{w}(r) \ge |C_{o_1} \cup \dots
\cup C_{o_i}|.
\]
The proof uses an inductive charging argument to assign each point in
$\bigcup_{j=1}^i C_{o_j}$ to a point in $\bigcup_{j=1}^i E_{x_j}$,
where each $r$ in the latter set will be in charge of at most
$\tilde{w}(r)$ points.  We define two charging rules. A point can be
either charged to its own proxy (\emph{Rule 1}) or to another point of
$T$ (\emph{Rule 2}).

Fix some arbitrary $i$, with $1 \leq i \leq k$, and assume,
inductively, that the points in $C_{o_1} \cup \dots \cup C_{o_{i-1}}$
have been charged to points in $\bigcup_{j=1}^{i-1} E_j$ for some
choice of distinct optimal centers $o_1, o_2, \ldots, o_{i-1}$. We
have two cases. \\
{\bf Case 1.} \emph{There exists an optimal center $o$ still unchosen
  such that there is a point $v \in C_{o}$ with 
$\pi_{\hat{\gamma}}(v) \in B_{x_j}$, for some $1 \leq j \leq i$.} 
We choose $o_i$ as one such center.  Hence
$d(x_j,\pi_{\hat{\gamma}}(v)) \le (1 + 2\epsilon) \cdot \rho$. By repeatedly
applying the triangle inequality we have that
for each $u \in C_{o_i}$
\begin{align*}
d(x_j,\pi_{\hat{\gamma}}(u)) \leq &
\;\; d(x_j,\pi_{\hat{\gamma}}(v)) + 
d(\pi_{\hat{\gamma}}(v),v) + 
d(v, o_i) \\  & +
d(o_i,u) + 
  d(u,\pi_{\hat{\gamma}}(u))  \leq (3 + 4\epsilon) \cdot \rho
\end{align*}
hence,  $\pi_{\hat{\gamma}}(u) \in E_{x_j}$. Therefore
we can charge each point $u\in C_{o_i}$ to its proxy, by Rule
1. \\
{\bf Case 2.}  \emph{For each unchosen optimal center $o$ and each $v
  \in C_o$, $\pi_{\hat{\gamma}}(v) \not\in \bigcup_{j=1}^i B_{x_j}$.}
We choose $o_i$ to be the unchosen optimal center which maximizes the
cardinality of $\{\pi_{\hat{\gamma}}(u) : u \in C_{o_i}\} \cap T'_i$.
We distinguish between points $u\in C_{o_i}$ with
$\pi_{\hat{\gamma}}(u) \notin T'_i$, hence $\pi_{\hat{\gamma}}(u) \in
\bigcup_{j=1}^{i-1} E_{x_j}$, and those with $\pi_{\hat{\gamma}}(u)
\in T'_i$.  We charge each $u\in C_{o_i}$ with $\pi_{\hat{\gamma}}(u)
\notin T'_i$ to its own proxy by Rule 1.  As for the other points, we
now show that we can charge them to the points of $B_{x_i}$.  To this
purpose, we first observe that $B_{\pi_{\hat{\gamma}}(o_i)}$ contains
$\{\pi_{\hat{\gamma}}(u) : u \in C_{o_i}\} \cap T'_i$, since for each
$u \in C_{o_i}$
\[
\begin{aligned}
  d(\pi_{\hat{\gamma}}(o_i),\pi_{\hat{\gamma}}(u)) 
  &\leq 
  d(\pi_{\hat{\gamma}}(o_i),o_i) + 
  d(o_i,u) + 
  d(u,\pi_{\hat{\gamma}}(u)) \\
  &\leq (1 + 2\epsilon) \cdot r^*_{k,z}(W) \leq (1 + 2\epsilon) \cdot \rho.
\end{aligned}
\]
Therefore the aggregate approximate weight of $B_{\pi_{\hat{\gamma}}(o_i)}$ is at least 
$\left|\left\{u \in C_{o_i} : \pi_{\hat{\gamma}}(u) \in T'_i\right\}\right|$.
Since Iteration $i$ selects $x_i$ as the 
center such that $B_{x_i}$ has maximum aggregate approximate weight,
we have that
\[
\sum_{r \in B_{x_i}} \tilde{w}(r) 
\ge \sum_{r \in B_{\pi_{\hat{\gamma}}(o_i)}} \tilde{w}(r) 
\ge \left|\left\{u \in C_{o_i} : \pi_{\hat{\gamma}}(u) \in T'_i\right\}\right|,
\]
hence, the points $u \in C_{o_i}$ with $\pi_{\hat{\gamma}}(u) \in T'_i$
can be charged to the points in $B_{x_i}$, since
$B_{x_i}$ has enough aggregate approximate weight. 

Note that the points of $B_{x_i}$ did not receive any charging by Rule
1 in previous iterations, since they are uncovered at the beginning of
Iteration $i$, and will not receive chargings by Rule 1 in subsequent
iterations, since $B_{x_i}$ does not intersect the set $C_o$ of any
optimal center $o$ yet to be chosen.  Also, no further charging to
points of $B_{x_i}$ by Rule 2 will happen in subsequent iterations,
since Rule 2 will only target sets $B_{x_h}$ with $h > i$. These
observations ensure that any point of $T$ receives charges through
either Rule 1 or Rule 2, but not both, and never in excess of its
weight, and the proof follows.
\end{proof}

At any time $t$, to obtain the desired solution
we invoke Procedure \proc{computeSolution}, which
works as follows (see
Algorithm~\ref{alg:computesolution} for the pseudocode).
The procedure first extracts the coreset $T=R_{\hat{\gamma}} \cup
O_{\hat{\gamma}}$ calling $\proc{extractCoreset}$. Then,
it sets  $\epsilon=4(1+\beta)$ and runs 
$\proc{outliersCluster}(T, k, \rho,
\epsilon)$ for a geometric sequence of values of $\rho$ between $\mindist$ and $\maxdist$, 
with step 
$1+\beta$, stopping at the minimum
value $\rho_{min}$ for which the pair $(X,T')$ returned by 
$\proc{outliersCluster}(T, k, \rho_{min}, \epsilon)$, is such that
aggregate approximate weight of set $T'$ is at most $z$\footnote{Note that
the parameter $\beta$ in the definitions of $\epsilon$ and of the step 
used in the geometric search for $\rho_{min}$ is the same one
 that appears in the definition of $\Gamma$, hence $\beta \in (0,1]$.} At this point,
the set of centers
$X$ is returned as solution to the $k$-center problem with $z$ outliers on the
current window $W$.

\begin{algorithm}[h] 
\small
    \SetAlgoLined
$T \gets \proc{extractCoreset}()$ \\
$\rho \gets \mindist$ \\
$\epsilon \gets 4(1+\beta)$ \\
$(X,T') \gets \proc{outliersCluster}(T, k, \rho, \epsilon)$ \\
$\tilde{w}(T') \gets \sum_{r \in T'} \tilde{w}(r)$ \\
\While{$\tilde{w}(T') > z$}{
$\rho \gets \rho \cdot (1+\beta)$ \\
$(X,T') \gets \proc{outliersCluster}(T, k, \rho, \epsilon)$ \\
$\tilde{w}(T') \gets \sum_{r \in T'} \tilde{w}(r)$ \\
}
\Return $X$
\caption{\proc{computeSolution}()}
\label{alg:computesolution}
\end{algorithm}

The following theorem highlights the tradeoff between accuracy (in
terms of both radius and excess number of outliers) and performance 
exhibited by $\proc{computeSolution}$.
\begin{Theorem} \label{thm:cssout_apprx}
\sloppy
At any time $t$, 
 $\proc{computeSolution}$ returns a set $X \subseteq W$ of at most $k$ centers 
  such that at least $|W|-(1+\lambda)z$ points of $W$ are
at distance at most $(23+55\beta)r^*_{k,z}(W)$ from $X$.  The
procedure requires a working memory of size $\BO{(k+z)\log_{1+\beta}
  (\maxdist/\mindist)\log_{1+\lambda}(|W|)}$ and runs in time
\[
\BO{\log_{1+\beta}(\maxdist/\mindist) \cdot k(k+z)^2}.
\]
If the while loop is substituted
by a binary search for  $\rho_{min}$, the running time decreases to
 \[
\BO{ \log(\log_{1+\beta}(\maxdist/\mindist)) \cdot k(k+z)^2}.
\]
\end{Theorem}
\begin{proof}
Let $\rho_{min}$ be such that 
$\proc{outliersCluster}(T, k,\rho_{min}, \epsilon)$
yields $(X,T')$, where $X$ is the returned solution, and let $W'$
be the set of points of $W$ whose proxies are in $T'$. By
Lemma~\ref{lem:outclust_apprx} and the choice of the step of the
geometric search for $\rho_{min}$, we have that $\rho_{min} \leq
(1+\beta)r^*_{k,z}(W)$, hence for each $p \in W-W'$,
$\distn(\pi_{\hat{\gamma}}(p),X) \leq (3+4\epsilon)\rho_{min} \leq
(3+4\epsilon)(1+\beta)r^*_{k,z}(W)$. 
Since $T$ is an $\epsilon$-coreset
(Theorem~\ref{thm:coreset}) and $\beta \in (0,1]$, we conclude  that 
for every $p \in W-W'$:
\begin{eqnarray*}
\distn(p,X) & \leq & \distn(p,\pi_{\hat{\gamma}}(p))+\distn(\pi_{\hat{\gamma}}(p),X) \\
& \leq & \epsilon r^*_{k,z}(W)+(3+4\epsilon)(1+\beta)r^*_{k,z}(W) \\
& < & (23+55\beta)r^*_{k,z}(W)
\end{eqnarray*}
Moreover, by 
Theorem~\ref{thm:coreset}, we also have that $|W'| \leq (1+\lambda) \sum_{r \in T'} \tilde{w}(r) \leq (1+\lambda)z$. 
The working memory bound is an immediate consequence of 
Theorem~\ref{thm:performance}, since
the working memory is dominated by the data structures from which
the coreset is extracted. For what concerns the running time,
observe that $\proc{outliersCluster}$ can be easily implemented in
to run in time $\BO{k\cdot |T|^2}$, which is $\BO{k \cdot (k+z)^2}$, since
$|T| = \BO{k+z}$ by Theorem~\ref{thm:coreset}. Therefore, the 
bound on the running time follows since $\proc{extractCoreset}$ requires time 
$\BO{ \log(\log_{1+\beta}(\maxdist/\mindist)) \cdot (k+z)^2}$
(by Theorem~\ref{thm:performance}),
and the while loop performs at most $\BO{\log_{1+\beta} (\maxdist/\mindist)}$
executions of $\proc{outliersCluster}$, which can be 
lowered to $\BO{\log\log_{1+\beta} (\maxdist/\mindist)}$ using binary search. 
\end{proof}
The above result shows that allowing for 
a slight excess in the number of outliers
(governed by parameter $\lambda$ 
results in improved space and time complexities. Note that if 
the upper bound $z$ on the number of outliers must be rigidly 
enforced, it is sufficient to set $\lambda = 1/(2z)$. In this case, 
$(1+\lambda)z = z + 1/2$, and since the number of outliers must be
an integer, it cannot be larger than $z$. The following corollary
is an immediate consequence of Theorem~\ref{thm:cssout_apprx}, 
of this observation, and of the fact that $\log_{1+\lambda} x = \BT{(1/\lambda)\log x}$, for $\lambda \in (0,1)$.
\begin{Corollary}\label{corol1}
Let $\lambda = 1/(2z)$. 
At any time $t$, 
 $\proc{computeSolution}$ returns a set $X \subseteq W$ of at most $k$ centers 
  such that at least $|W|-z$ points of $W$ are
at distance at most $(23+55\beta)r^*_{k,z}(W)$ from $X$.  The
procedure requires a working memory of size $\BO{z(k+z)\log_{1+\beta}
  (\maxdist/\mindist)\log (|W|)}$ and runs in time
\[
\BO{\log_{1+\beta}(\maxdist/\mindist) \cdot k(k+z)^2}.
\]
If the while loop is substituted
by a binary search for  $\rho_{min}$, the running time decreases to
 \[
\BO{ \log(\log_{1+\beta}(\maxdist/\mindist)) \cdot k(k+z)^2}.
\]
\end{Corollary}

\subsection{Obliviousness to $\mindist$ and $\maxdist$} \label{sec:obliviousness}
The algorithm described in Subsections \ref{sec:coreset} and
\ref{sec:solution} requires the knowledge of the values $\mindist$ and
$\maxdist$.  In this subsection, we show how to remove this
requirement by employing the techniques developed in
\cite{PellizzoniPP20}, suitably extended to cope with histograms, which
were not used in that work.

Let $p_1, p_2, \ldots $ be an enumeration of all points of the stream
$S$, based on their arrival times. For $t > k+z$, let $d_t$ be the
minimum pairwise distance between the last $k+z+1$ points of the
stream ($p_{t-k-z}, \ldots, p_{t-1}, p_t$). Let also $D_t$ be the
maximum distance between $p_1$ and any $p_i$ with $i \leq t$, and note
that, by the triangle inequality, the maximum pairwise distance among
the first $t$ points of $S$ is upper bounded by $2D_t$. It is easy to argue that $d_t/2 \leq r^*_{k+z}(W) \leq r^*_{k,z}(W) \leq 2D_t$. 
By storing $p_1$ and the last $k+z+1$ points of the stream,
the values $d_t$ and $D_t$ can be straightforwardly maintained
by the algorithm with $\BO{(k+z)^2}$ operations per step. We define
\[
\Gamma_t = \{(1+\beta)^i : 
\lfloor \log_{1+\beta} d_t/2 \rfloor 
\leq i \leq 
\lceil \log_{1+\beta} 2D_t \rceil
\}.
\]
Suppose that at any time $t$ the algorithm maintains the sets
$A_{\gamma}$, $R_{\gamma}$ and $O_{\gamma}$, and the histograms for
the points in $R_{\gamma} \cup O_{\gamma}$, only for $\gamma \in
\Gamma_t$, and assume that the properties stated in
Lemmas~\ref{lem:technical} and \ref{lem:cleanup} hold for every
$\gamma \in \Gamma_t$ and $r \in R_{\gamma} \cup O_{\gamma}$.  Then,
by running procedure $\proc{extractCoreset}$, limiting the search for
$\hat{\gamma}$ to the set $\Gamma_t$, we still obtain a
$4(1+\beta)$-coreset for the current window. To see this, we first
note that, based on the previous observation, $\Gamma_t$ includes for
sure a value $\gamma$ with $r^*_{k+z}(W) \leq \gamma \leq
(1+\beta)r^*_{k+z}(W)$. By repeating the same argument used in the
proof of Theorem~\ref{thm:coreset}, we can show that for such a value
of $\gamma$ we have that $|A_{\gamma}| \leq k+z$, and that the inner
for loop of $\proc{extractCoreset}$ computes a set $C$ of at most
$k+z$ points. This immediately implies that $\proc{extractCoreset}$
determines a guess $\hat{\gamma} \leq (1+\beta)r^*_{k+z}(W)$ and that the
returned coreset $T=R_{\hat{\gamma}} \cup O_{\hat{\gamma}}$ is a
$4(1+\beta)$-coreset.

We now show how to modify the algorithm described in the previous subsections 
(referred to as \emph{full algorithm} in what follows) to maintain, without the
knowledge of $\mindist$ and $\maxdist$, the sets $A_{\gamma}$,
$R_{\gamma}$ and $O_{\gamma}$ and the required histograms, for every
guess $\gamma \in \Gamma_t$. Suppose that this is the case up to some
time $t-1 > k+z$, and consider the arrival of $p_t$. Before invoking
$\proc{update}(p_t,t)$, the algorithm executes the 
operations described below. 

First, the new values $d_t$ and $D_t$ are computed, and all sets
relative to values of $\gamma \in \Gamma_{t-1}-\Gamma_t$ are removed.
If $d_t < d_{t-1}$, then for each $\gamma \in \Gamma_t$ with $\gamma <
\min \{\gamma' \in \Gamma_{t-1}\} \leq d_{t-1}/2$, the algorithm sets
$A_\gamma = \{p_{t-k-z-1}, \ldots, p_{t-1}\} = R_\gamma$ and $O_\gamma
= \emptyset$. Moreover, for each $p_{\tau} \in R_{\gamma}$, it sets
$L_{p_{\tau}} = \{ (\tau, 1) \}$, as each point represents itself
only. Since any two points in $A_{\gamma}$ are at distance at least
$d_{t-1} > 2\gamma$, it is easy to see that these newly created data
structures coincide with the ones that the full algorithm would store
at time $t-1$ (for the same $\gamma$'s) if the stream started at time
$t-(k+z+1)$, hence they satisfy the properties of
Lemmas~\ref{lem:technical} and \ref{lem:cleanup}.

If $D_t > D_{t-1}$, then for each $\gamma \in \Gamma_t$ with $\gamma >
\max \{\gamma' \in \Gamma_{t-1}\}$, the algorithm sets $A_\gamma =
\{p_{t-|W|}\}$, $R_\gamma = \{p_{t-1}\}$ and $O_\gamma = \emptyset$.
It is easy to see that these newly created sets
coincide with the ones that the full algorithm would store at time
$t-1$ (for the same $\gamma$'s) if the stream started at time
$t-|W|$, hence they satisfy the properties of
Lemma~\ref{lem:technical}.
It has to be remarked that, although point $p_{t-|W|}$ is
not available at time $t-1$, this is not a problem since the point
immediately expires at time $t$ and is removed from the data
structures without even being used in the processing of $p_t$. Hence,
in this context, $p_{t-|W|}$ acts as a mere placeholder. For
what concerns the histogram $L_{p_{t-1}}$ to associate with
$p_{t-1}$ (for every  $\gamma >
\max \{\gamma' \in \Gamma_{t-1}\}$), its exact version 
represents the entire active window, hence it 
would be the list $\{ (t-|W|, |W|), \ldots, (t-2, 2), (t-1, 1) \}$. 
In order to satisfy the properties of Lemma~\ref{lem:cleanup},
$L_{p_{t-1}}$ can be trimmed as follows. 
Let $f(x) = \left\lceil \frac{x}{1+\lambda} \right\rceil$.
Then, $L_{p_{t-1}}$ contains the set of pairs
$(t-c_i,c_i)$, with $i\geq0$, where the $c_i$'s form a decreasing sequence
of values $\geq 1$ such that $c_0 = |W|$ and, for each $i \leq 0$ with $c_i>1$
$c_{i+1} = \min\{c_i-1, f(c_i)\}$.
\begin{Lemma}
The list $L_{p_{t-1}}$ defined above satisfies the properties stated in Lemma~\ref{lem:cleanup}.
\end{Lemma}
\begin{proof}
Property~\ref{prop1} clearly holds. As for Property~\ref{prop2}, consider two
consecutive pairs $(t-c_i,c_i)$ and $(t-c_{i+1},c_{i+1})$. 
If $c_{i+1} = f(c_i)$, then for sure $c_i \leq (1+\lambda)c_{i+1}$.
Hence, when  $c_i > (1+\lambda)c_{i+1}$ we must have
$c_{i+1}=c_i-1$, and the property follows. 
Property~\ref{prop3} holds since, for $1 \leq i \leq |L_r|-2$, 
$c_{i+2} \leq c_{i+1}- 1 < (c_i/(1+\lambda) + 1) - 1$,
hence $(1+\lambda) c_{i+2} \leq c_i$. 
Finally, the bound stated by Property~\ref{prop4} follows as a direct
consequence of the first three properties.
\end{proof}

Once the correct configurations of the data structures for all guesses
$\gamma \in \Gamma_t$ are obtained, Procedure $\proc{update}(p_t, t)$
is invoked to complete step $t$, thus enforcing the properties of
Lemmas~\ref{lem:technical} and \ref{lem:cleanup}, for all of these
data structures.

\subsection{Improved Approximation under Bounded Doubling Dimension} \label{sec:ddimension}
Consider a stream $S$ of doubling dimension $D$. We now
outline an improved coreset construction that is able to provide a
$\delta$-coreset $T$ for the $k$-center problem with $z$ outliers, for
\emph{any} given $\delta>0$, at the expense of a blow-up in the
working memory size which is analyzed as a function of $D$ and is
tolerable for small (e.g., constant) $D$. This improved construction
allows us to obtain a much tighter approximation for the $k$-center
problem with $z$ outliers in the sliding window setting.

Fix any given $\delta>0$. In \cite{PellizzoniPP20}, a refinement of
the $k$-center strategy of \cite{Cohen} is presented which, for every
guess $\gamma$, maintains two families of attraction, representative
and orphan points. The first family, referred to as \emph{validation
  points}, features three $\BO{k}$-size sets of attraction,
representative and orphan points, equivalent to those
described in Subection \ref{sec:coreset}. Validation points are 
 employed to identify a constant approximation $\hat{\gamma}$ to
the optimal radius $r^*_k(W)$.  The second family, referred to as
\emph{coreset points}, contains, for any guess $\gamma$,  three “expanded'' sets of attraction,
representative and orphan points, which refine the coverage provided by the corresponding sets of
validation points, in the sense that the coreset points
relative to the guess $\hat{\gamma}$ yield a coreset $T$ such that $\max_{p \in W} d(p,T)
\leq \delta r^*_k(W)$.  For each $\gamma$, these larger sets contain
$O(k(c/\delta)^D)$ points, for a suitable constant $c$.

We can augment the algorithm of \cite{PellizzoniPP20} in the same fashion
as we agumented the algorithm of \cite{Cohen},
by endowing the representative and orphan coreset points with the
histograms described in Subsection \ref{sec:coreset}. Then,
by running the resulting algorithm for $k+z$ (instead of $k$) centers, 
the result stated in the following lemma is immediately
obtained, where parameter $\lambda$ and functions $w(\cdot)$ and
$\tilde{w}(\cdot)$ have the same meanings as before. 

\begin{Lemma} \label{thm:bettercoreset}
Let $\delta, \lambda>0$ be two design parameters. For a stream $S$ of
doubling dimension $D$, suitable data structures can be maintained
from which, at any time $t$, a weighted $\delta$-coreset $T$ of size
$O((c/\delta)^D(k+z))$, for a fixed constant $c$, can be extracted such that,
for each $p \in W$, there exists a proxy $\pi(p) \in T$  with
\[
\distn(p,\pi(p)) \leq \delta r^*_{k+z}(W) \leq \delta r^*_{k,z}(W).
\]
Moreover, for each $r \in T$, an approximate weight $\tilde{w}(r)$
can be computed, with $w(r)/(1+\lambda) \leq \tilde{w}(r) \leq w(r)$,
where $w(r) = |\{p \in W : \pi(p)=r\}|$. The data
structures require a working memory of size
$\BO{(c/\delta)^D(k+z)\log (\maxdist/\mindist)\log_{1+\lambda}(|W|)}$.
\end{Lemma}
The analysis in \cite{PellizzoniPP20} implies that the constant $c$
can be fixed arbitrarily close to 32.  We remark that the construction
described in Subsection~3.1 cannot return
$\delta$-coresets with $\delta \leq 4$, while the result of
Lemma~\ref{thm:bettercoreset} yields $\delta$-coresets for any
$\delta>0$.

To obtain the desired solution at any time $t$, we first compute a
$\delta$-coreset $T$ satisfying the properties stated in the above
lemma. Then, analogously to $\proc{computeSolution}$,
we run $\proc{outliersCluster}(T, k, \rho, \delta)$ for a
geometric sequence of values of $\rho$ of step $1+\beta$ between
$\mindist$ and $\maxdist$, stopping at the minimum value $\rho_{min}$
for which, if $(X,T')$ is the output of $\proc{outliersCluster}(T,
k, \rho_{min}, \delta)$, then the aggregate approximate weight of set
$T'$ is at most $z$. The algorithm returns $X$ as the final set of
(at most) $k$ centers.
By choosing $\beta = \delta/(3+4\delta)$, we get the the following result:
\begin{Theorem} \label{thm:pppout_apprx}
Let $\delta, \lambda >0$ be two design parameters and consider a
stream of doubling dimension $D$.  There exists a sliding window
algorithm that, at any time $t$, returns a set of at most $k$ centers
$X \subseteq W$ such that at least $|W|-(1+\lambda)z$ points of $W$
are at distance at most $(3+6\delta)r^*_{k,z}(W)$ from $X$.  For a suitable constant $c>0$,
the algorithm makes use of a working memory of size
$O((c/\delta)^D(k+z)\log (\maxdist/\mindist)
  \log_{1+\lambda}(|W|))$. Also, the algorithm requires time
 \[ 
 \BO{\log_{1+\beta}(\maxdist/\mindist) \cdot \left( (c/\delta)^D(k+z) + \log_{1+\lambda}(|W|)\right) }
 \]
to update the data structures after each point arrival, and time 
 \[
\BO{ \log\left(\log_{1+\beta}(\maxdist/\mindist)\right) \cdot ((c/\delta)^D\cdot k(k+z))^2}
\]
to compute the final solution, with
$\beta = \delta/(3+4\delta)$.
\end{Theorem}
\begin{proof}
By reasoning 
as in the proof of Lemma~\ref{lem:outclust_apprx}, we can show
that, by executing $\proc{outliersCluster}(T, k, \rho, \delta)$, 
with any $\rho \geq r^*_{k,z}(W)$, the set of uncovered points 
at the end of the execution has
aggregate approximate weight at most $z$. 
This immediately implies that  $\rho_{min} \leq (1+\beta)r^*_{k,z}(W)$.
Let $(X,T')$ be the output of $\proc{outliersCluster}(T, k,
\rho_{min}, \delta)$, and let $W'$
be the set of points of $W$ whose proxies are in $T'$. We have that
for each $p \in W-W'$,
$\distn(\pi(p),X) \leq (3+4\delta)\rho_{min} \leq
(3+4\delta)(1+\beta)r^*_{k,z}(W)$. 
Since $T$ is an $\delta$-coreset we conclude that,
for every $p \in W-W'$:
\begin{eqnarray*}
\distn(p,X) & \leq & \distn(p,\pi(p))+\distn(\pi(p),X) \\
& \leq & \delta r^*_{k,z}(W)+(3+4\delta)(1+\beta)r^*_{k,z}(W) \\
& \leq & (3+6\delta)r^*_{k,z}(W),
\end{eqnarray*}
where the last inequality uses the fact that $\beta=\delta/(3+4\delta)$.
Moreover, we also have that $|W'| \leq (1+\lambda) \sum_{r \in T'} \tilde{w}(r) \leq (1+\lambda)z$. 
Finally, the bound on the working memory follows from
Lemma~\ref{thm:bettercoreset}, while the time bounds for updating the data structures and extracting the solution from the coreset 
are obtained by adapting the arguments
used to prove Theorems~\ref{thm:performance} and \ref{thm:cssout_apprx}. \end{proof}

The following corollary is the counterpart of Corollary~\ref{corol1} for the dimension-sensitive algorithm developed in this section. 
\begin{Corollary}
Let $\lambda = 1/(2z)$. 
At any time $t$,  the algorithm is able to compute a set $X \subseteq W$ of at most $k$ centers 
  such that at least $|W|-z$ points of $W$ are
at distance at most $(3+6\delta)r^*_{k,z}(W)$ from $X$.  For a fixed constant $c>0$, the
algorithm requires a working memory of size $\BO{(c/\delta)^Dz(k+z)\log (\maxdist/\mindist)\log (|W|)}$ and runs in time
 \[
\BO{ \log(\log_{1+\beta}(\maxdist/\mindist)) (c/\delta)^D \cdot k(k+z)^2},
\]
with $\beta=\delta/(3+4\delta)$.
\end{Corollary}

It is important to remark that the algorithm in \cite{PellizzoniPP20},
which we have built upon, is \emph{fully oblivious} to $D$, $\maxdist$
and $\mindist$. Its augmentation discussed above inherits the
obliviousness to $D$ straightforwardly, while the obliviousness to
$\maxdist$ and $\mindist$ is inherited through the technique described
in Section~\ref{sec:obliviousness}.

\section{Effective Diameter Estimation} \label{sec:effdiameter}
Consider a stream $S$ of doubling dimension $D$. 
Building on the improved coreset
construction of Lemma~\ref{thm:bettercoreset}, 
we now outline an algorithm that, at any time $t$, 
is able to compute lower and upper estimates of the 
$\alpha$-effective diameter $\Delta^{\alpha}_W$ of the current window
$W$. The algorithm requires the knowledge of a (possibly crude) lower
bound $\eta \in (0,1)$ on the ratio between $\Delta^{\alpha}_W$ and
the diameter $\Delta_W$ (i.e., $\Delta_W^{\alpha} \geq \eta\Delta_W$).

For a given
$\epsilon>0$, let $T$ be a weighted $\delta$-coreset for $W$ computed with the
properties stated in Lemma~\ref{thm:bettercoreset}, with $\delta =
\epsilon\eta/2$, $k=1$ and $z=0$.  Hence $|T| =
O((c'/(\epsilon\eta))^D)$, with $c'=2c$, where $c$ is the same constant appearing the
statement of the lemma.  Assume for now
that for each $r\in T$, the true weight $w(r) = |\{p \in W :
\pi(p)=r\}|$ is known.  (Later we will discuss the distortion
introduced by using the approximate weights $\tilde{w}(r)$.)  An
approximation to $\Delta_W^{\alpha}$ can be computed on the coreset
through the following quantity:
\[
\Delta^{\alpha}_{T,W} =
\argmin_d 
\left\{
\sum_{\stackrel{r_1,r_2 \in T:}{\footnotesize \distn(r_1,r_2) \leq d}} 
\!\!\!\! w(r_1)w(r_2) \geq \alpha |W|^2
\right\}.
\]
\begin{Lemma} \label{proposition:effective}
$(1-\epsilon) \Delta_W^{\alpha} \leq \Delta^{\alpha}_{T,W} \leq 
(1+\epsilon)  \Delta_W^{\alpha}.
$
\end{Lemma}
\begin{proof}
Let $\hat{d} = \max\{\distn(p,\pi(p)) : p \in W\}$.
Recall that the $\delta$-coreset $T$ was computed with $k=1$ and $z=0$,
hence, by the properties of $T$ and the fact that
$\Delta_W^{\alpha} \geq \eta\Delta_W$, we have that
\[
\hat{d} \leq \delta r_1^*(W) \leq (\epsilon \eta/2) \Delta_W \leq
(\epsilon/2) \Delta_W^{\alpha}.
\]
By 
the triangle inequality, for any $d$
and any pair $(p,q) \in W \times W$, we have
that if $\distn(\pi(p),\pi(q)) \leq d-2\hat{d}$ then 
$\distn(p,q) \leq d$. Thus, when 
$|\{(p,q) \in W \times W : \distn(\pi(p),\pi(q)) \leq d-2\hat{d}\}| 
\geq \alpha |W|^2$,
we must also have
$|\{(p,q) \in W \times W : \distn(p,q) \leq d\}| \geq \alpha |W|^2$.
Consequently, 
\[
\begin{split}
\Delta_W^{\alpha} & =  
\argmin_d \left\{
|\{(p, q) \in W \times W : \distn(p, q) \leq d \}| \geq \alpha |W|^2\right\} \\
&
\leq 
\argmin_d  
\left\{|\{(p, q) \in W \times W : \distn(\pi(p), \pi(q)) \leq d -2\hat{d}\}| 
\geq \alpha |W|^2\right\} \\
& = 
\argmin_d \left\{\sum_{\stackrel{r_1,r_2 \in T:}{\footnotesize \distn(r_1,r_2) \leq d -2\hat{d}}} 
\!\!\!\! w(r_1)w(r_2) \geq \alpha |W|^2\right\} \\
& = 
\argmin_d \left\{\sum_{\stackrel{r_1,r_2 \in T:}{\footnotesize \distn(r_1,r_2) \leq d}} 
\!\!\!\! w(r_1)w(r_2) \geq \alpha |W|^2\right\}+2\hat{d} \\
& =
\Delta^{\alpha}_{T,W} + 2\hat{d} \leq
\Delta^{\alpha}_{T,W} +\epsilon \Delta_W^{\alpha},
\end{split}
\]
thus proving the first stated inequality. The proof of the other inequality
is accomplished with a symmetrical argument.
The triangle inquality ensures that
for every pair $(p,q) \in W \times W$, 
if
$\distn(p,q) \leq d$ then $\distn(\pi(p),\pi(q)) \leq d+2\hat{d}$. 
Thus, when 
$|\{(p,q) \in W \times W : \distn(p,q) \leq d\}| \geq \alpha |W|^2$
we must also have
$|\{(p,q) \in W \times W : \distn(\pi(p),\pi(q)) \leq d+2\hat{d}\}| 
\geq \alpha |W|^2$. 
Consequently, 

\[
\begin{split}
\Delta_W^{\alpha} & = 
\argmin_d \left\{|\{(p, q) \in W \times W : \distn(p, q) \leq d \}| \geq \alpha |W|^2\right\} \\
&
\geq 
\argmin_d 
\left\{|\{(p, q) \in W \times W : \distn(\pi(p), \pi(q)) \leq d+2\hat{d}\}| \geq \alpha |W|^2\right\} \\
& = 
\argmin_d \left\{\sum_{\stackrel{r_1,r_2 \in T:}{\footnotesize \distn(r_1,r_2) \leq d+2\hat{d}}} 
\!\!\!\! w(r_1)w(r_2) \geq \alpha |W|^2\right\} \\
& = 
\argmin_d \left\{\sum_{\stackrel{r_1,r_2 \in T:}{\footnotesize \distn(r_1,r_2) \leq d}} 
\!\!\!\! w(r_1)w(r_2) \geq \alpha |W|^2\right\}-2\hat{d} \\
& =
\Delta^{\alpha}_{T,W} - 2\hat{d} \geq
\Delta^{\alpha}_{T,W} -\epsilon \Delta_W^{\alpha}.
\end{split}
\]
\end{proof}
Recall now that for every coreset point $r \in T$, only an approximation
$\tilde{w}(r)$ to the actual weight $w(r)$ is available, with
$w(r)/(1+\lambda) \leq \tilde{w}(r) \leq w(r)$.  We define the approximate counterpart of  ${\Delta}^{\alpha}_{T,W}$ as
\[
\tilde{\Delta}^{\alpha}_{T,W} =
\argmin_d \!\!\!\!\sum_{\stackrel{r_1,r_2 \in T:}{\footnotesize \distn(r_1,r_2) \leq d}} 
\!\!\!\! \tilde{w}(r_1)\tilde{w}(r_2) \geq \alpha |W|^2.
\]
Our approximation algorithm
returns $(1/(1+\epsilon))\tilde{\Delta}^{\alpha/(1+\lambda)^2}_{T,W}$ 
and $(1/(1-\epsilon))\tilde{\Delta}^{\alpha}_{T,W}$ as
lower and upper estimates, respectively, of the true effective diameter 
$\Delta^{\alpha}_W$. The following theorem
establishes the tightness of these estimates and the space and time
performance of the algorithm. 

\begin{Theorem}\label{thm:effdiameter}
Consider a stream $S$ of doubling dimension $D$, and a value $\alpha
\in (0,1)$.  Suppose that a value $\eta < 1$ is known such that, for
every window $W$, $\Delta^{\alpha}_W \geq \eta \Delta_W$. For any
$\epsilon, \lambda > 0$, there exists a sliding window algorithm that,
at any time $t$ is able to compute a weighted coreset $T$ of size
$O((c'/(\epsilon \eta))^D)$ such that
\[
\frac{1}{1+\epsilon}
\tilde{\Delta}^{\alpha/(1+\lambda)^2}_{T,W} \leq \Delta^{\alpha}_W \leq 
\frac{1}{1-\epsilon}
\tilde{\Delta}^{\alpha}_{T,W},
\]
where $W$ is the current window and $c'>0$ is a suitable constant.  The
algorithm makes use of a working memory of size $\BO{(c'/(\epsilon \eta))^D\log
  (\maxdist/\mindist) \log_{1+\lambda}(|W|)}$. Also, the algorithm
requires time
\[ 
\BO{\log_{1+\beta}(\maxdist/\mindist) \cdot \left( (c'/(\epsilon \eta)) + \log_{1+\lambda}(|W|)\right) }
\]
to update the data structures after each point arrival,
where $\beta = \delta/(3+4\delta)$. The
lower and upper estimates to $\Delta^{\alpha}_W$ can be computed from $T$
in time $\BO{|T|^2} =  \BO{(c'/(\epsilon \eta))^{2D}}$.
\end{Theorem}
\begin{proof}
We first prove that
\[
\tilde{\Delta}^{\alpha/(1+\lambda)^2}_{T,W} 
\leq \Delta^{\alpha}_{T,W} \leq 
\tilde{\Delta}^{\alpha}_{T,W}.
\]
Then, the stated approximation interval will immediately follow by 
Lemma~\ref{proposition:effective}. 
Let us first prove the leftmost inequality. From Lemma~\ref{thm:bettercoreset}
we have that for every $r \in T$, $\tilde{w}(r) \geq w(r)/(1+\lambda)$. Hence,
for any $d$ such that
\[
\sum_{\stackrel{r_1,r_2 \in T:}{\footnotesize \distn(r_1,r_2) \leq d}} 
w(r_1)w(r_2) \geq \alpha |W|^2,
\]
we have that
\[
\sum_{\stackrel{r_1,r_2 \in T:}{\footnotesize \distn(r_1,r_2) \leq d}} 
\tilde{w}(r_1)\tilde{w}(r_2) \geq \frac{\alpha}{(1+\lambda)^2} |W|^2,
\]
which implies $\tilde{\Delta}^{\alpha/(1+\lambda)^2}_{T,W} 
\leq \Delta^{\alpha}_{T,W}$.
The righmost inequality is proved in a simmetrical fashion. 
Again, from Lemma~\ref{thm:bettercoreset}
we have that for every $r \in T$, $w(r) \geq \tilde{w}(r)$. Hence,
for any $d$ such that
\[
\sum_{\stackrel{r_1,r_2 \in T:}{\footnotesize \distn(r_1,r_2) \leq d}} 
\tilde{w}(r_1)\tilde{w}(r_2)\geq \alpha |W|^2,
\]
we have that
\[
\sum_{\stackrel{r_1,r_2 \in T:}{\footnotesize \distn(r_1,r_2) \leq d}} 
w(r_1)w(r_2) \geq \alpha |W|^2,
\]
which implies $\Delta^{\alpha}_{T,W} \leq \tilde{\Delta}^{\alpha}_{T,W}$. 
The bounds on the working memory and on the update time follow directly
from Theorem~\ref{thm:pppout_apprx} and from the choice of $k=1$ and $z=0$.
Finally, the estimates $\tilde{\Delta}^{\alpha/(1+\lambda)^2}_{T,W}$
and $\tilde{\Delta}^{\alpha}_{T,W}$ can be computed from $T$ in
time $\BO{|T|^2} =  \BO{(c'/(\epsilon \eta))^{2D}}$, using a simple 
strategy based on binary search.
\end{proof}

The theorem implies that by setting $\epsilon$ and $\lambda$
sufficiently small, we can get tight estimates for
$\Delta^{\alpha}_W$, for all windows for which the value of the
$\alpha$-effective diameter behaves smoothly in an interval to the
left of $\alpha$.  Finally, we remark that the algorithm is fully
oblivious to $D$, $\mindist$, and $\maxdist$. Moreover, while the
theoretical space bound exhibits a dependency on $(1/\eta)^D$, in the
next section we provide experimental evidence of a much lesser impact
of $\eta$ for datasets where outliers represent true noise, proving
that the working space requirements exhibit a milder dependence on the
crudeness of the lower bound $\eta$.


\section{Experiments} \label{sec:experiments}

We implemented the algorithms for $k$-center with $z$ outliers and for
the estimation of the effective diameter presented in
Sections~\ref{sec:algorithm} and~\ref{sec:effdiameter}.  For what
concerns $k$-center with $z$ outliers, we implemented the
dimensionality-sensitive algorithm of Subsection~\ref{sec:ddimension},
which offers a wider spectrum of performance-accuracy tradeoffs.
We ran proof-of-concept experiments aimed at testing the algorithms'
behavior against relevant competitors, in terms of approximation,
memory usage, and running time for processing each point arrival
(\emph{update time}) and for computing a solution for the current
window, whenever needed (\emph{query time}).  All tests were executed
using Java 13 on a Windows machine running on an AMD FX8320 processor
with 12GB of RAM, with the running times measured using
\texttt{System.nanoTime}, and feeding the points to the algorithms
through the file input stream.

\subsection{$k$-center with outliers}

Along with our algorithm (dubbed \coresetOutliers) we implemented
\footnote{The source code and the datasets used in our experiments
are provided on GitHub at \texttt{github.com/PaoloPellizzoni/OutliersSlidingWindows}}
the sequential $3$-approximation by \cite{Charikar2001}
(dubbed \charikar), to be run on the entire window $W$, consisting
of a search for the minimum $\rho$ such that
$\proc{outliersCluster}(W, k, \rho, 0)$, run with unit weights, ends with
at most $z$ uncovered points. We chose this sequential benchmark over the
existing LP-based 2-approximation algorithms, since these latter algorithms do not seem to admit
practical implementations.  Since \charikar\ itself is plagued by a
superquadratic complexity, which makes it
unfeasible for larger windows, we also devised a sampled version
(dubbed \sampledCharikar) where the center selection in each call to
$\proc{outliersCluster}$ examines only a fixed number of random candidates,
rather than all window points. We deemed unnecessary to perform
a comparison of our algorithm with the one \cite{DeBergMZ21}, 
since, as mentioned in the introduction, this latter algorithm needs to run an instance of 
$\proc{outliersCluster}$ for each update operation, and would thus prove to be 
a poor competitor of our strategy, where the execution 
of this expensive sequential procedure is confined only to the query operation.

Also, to asses the importance of using
a specialized algorithm to handle outliers, we compared the quality of
our solution against the one returned by the algorithm
of  \cite{GONZALEZ1985293} for $k$-center without outliers (dubbed \gon),
where the radius is computed excluding the $z$ largest distances from
the centers.

The algorithms were tested on the following datasets, often used in
previous works \cite{MalkomesKCWM15}: the \texttt{Higgs}
dataset\footnote{http://archive.ics.uci.edu/ml/datasets/HIGGS}, which
contains 11 million 7-dimensional points representing
high-energy particle features generated through Monte-Carlo
simulations; and the \texttt{Cover}
dataset\footnote{https://archive.ics.uci.edu/ml/datasets/covertype},
which contains $581012$ 55-dimensional points from geological
observations of US forest biomes, and was employed as a stress test
for our dimensionality-sensitive algorithm.  We also generated inflated versions of the original 
datasets, dubbed \texttt{Higgs+} and \texttt{Cover+}, by artificially
injecting a new true outlier point after each original point with
probability $p$, where the new point has norm 100 times the diameter
of the original dataset (e.g. as if produced by a malfunctioning
sensor). The probability $p$ was chosen to yield $z/2$ true outliers per window,
in expectation.
We performed tests for $k=10$, $z=10, 50$, and window sizes
$|W|=N \in \{10^4,10^5, 10^6\}$, using Euclidean distance. 

For \coresetOutliers, we set
$\delta = 2/3$, $\beta = 0.5$ and $\lambda = 0.5$; moreover, we set $\mindist = 0.01$ and 
$\maxdist =10^4$, which are conservative lower and upper estimates of the
clustering radii for all windows and all datasets.  
The implementations of $\charikar$ and $\sampledCharikar$
execute, for a window $W$, a search for a minimum $\rho$ such that
$\proc{outliersCluster}(W, k, \rho, 0)$ with unit weights, ends with
$\leq z$ uncovered points.  For $\sampledCharikar$,
$\proc{outliersClusters}$ has been modified so that each new center is
selected among a set of random window points of expected size 1000.

Tables~\ref{suppl-tab-1} and \ref{suppl-tab-2} detail the full results
of the experiments on $k$-center clustering with $z$ outliers.  All
quantities are provided as one-sigma confidence intervals, based on 10
windows sampled every $10^4$ timesteps after the first $N$ insertions.
Starred results are based on a single sample due to the excessively
high running time.  Table~\ref{suppl-tab-1} reports the average ratio
between the clustering radius obtained by each tested algorithm and
the one obtained by \coresetOutliers, as well as the average number of
floats maintained in memory by the algorithms. (All radii have been
computed with respect to the entire window, excluding the $z$ largest
distances from the centers.)  As shown in the table,
\coresetOutliers\ is always within a few percentage points from the
radius of the solution of \charikar\ (which did not finish in
reasonable time for $N=10^6$). On the other hand, the quality of the
solution of \sampledCharikar\ degrades as the window size grows, since
the fraction of center candidates decreases.  Moreover, as expected,
\gon\ yields poorer performance especially in presence of true
outliers, as these are mistakenly selected as centers instead of being
disregarded.  Hence, \gon\ is not considered in the successive
experiments.  Figure \ref{fig-memory} plots the memory usage (in
floats) of the algorithms for \texttt{Higgs}, confirming that the
working memory required by \coresetOutliers\ grows sublinearly with
$N$, and it is much smaller than the one required by \charikar\ and
\sampledCharikar, which is linear in $N$.

\begin{table*}
\centering
\scriptsize
\begin{tabular}{ c c  c c c  c  c c c}
 \hline 
 \multirow{3}{*}{Dataset} & \multirow{3}{*}{Algorithm} & \multicolumn{3}{c}{Obj. ratio} & \null & \multicolumn{3}{c}{Memory ($\times 10^6$ floats)} \\
 \cline{3-5} \cline{7-9}
 \null & \null & \multicolumn{3}{c}{Window size} & \null & \multicolumn{3}{c}{Window size}     \\
 \null & \null & $10^4$ & $10^5$ & $10^6$ &  & $10^4$ & $10^5$ & $10^6$   \\
 \hline  \hline
 \multirow{4}{1cm}{\textsc{Higgs} (z=10)} & \coresetOutliers  & $1 \pm 0$ & $1 \pm 0$ & $1 \pm 0$ & & $0.13 \pm 0.01$ & $0.27 \pm 0.02$ & $0.42 \pm 0.02$  \\
 \null                           & \charikar                  & $1.02 \pm 0.05$ & $0.99 \pm 0.04$ & -- & & $0.07 \pm 0$ & $0.7 \pm 0$ & $7 \pm 0$ \\
 \null                           & \sampledCharikar           & $1.07 \pm 0.05$ & $1.43 \pm 0.23$ & $2.74 \pm 0.7$ & & $0.07 \pm 0$ & $0.7 \pm 0$ & $7 \pm 0$ \\
 \null                           & \gon                       & $1.18 \pm 0.12$ & $1.17 \pm 0.08$ & $1.05 \pm 0.03$ & & $0.07 \pm 0$ & $0.7 \pm 0$ & $7 \pm 0$ \\
 \hline
 \multirow{4}{1cm}{\textsc{Higgs} (z=50)} & \coresetOutliers  & $1 \pm 0$ & $1 \pm 0$ & $1 \pm 0$ & & $0.26 \pm 0.01$ & $0.65 \pm 0.03$ & $1.25 \pm 0.02$ \\
 \null                           & \charikar                  & $1.02 \pm 0.04$ & -- & -- & & $0.07 \pm 0$ & $0.7 \pm 0$ & $7 \pm 0$ \\
 \null                           & \sampledCharikar           & $1.04 \pm 0.06$ & $1.14 \pm 0.08$ & $1.59 \pm 0.15$ & & $0.07 \pm 0$ & $0.7 \pm 0$ & $7 \pm 0$ \\
 \null                           & \gon                       & $1.54 \pm 0.19$ & $1.5 \pm 0.15$ & $1.19 \pm 0.06$ & & $0.07 \pm 0$ & $0.7 \pm 0$ & $7 \pm 0$ \\
 \hline
 \multirow{4}{1cm}{\textsc{Cover} (z=10)} & \coresetOutliers  & $1 \pm 0$ & $1 \pm 0$ &   & & $0.72 \pm 0.18$ & $2.06 \pm 0.13$ & \\
 \null                           & \charikar                  & $0.97 \pm 0.17$ & $1.02^*$ &   & & $0.55 \pm 0$ & $5.5 \pm 0$ &   \\
 \null                           & \sampledCharikar           & $0.96 \pm 0.17$ & $0.98 \pm 0.18$ &   & & $0.55 \pm 0$ & $5.5 \pm 0$ &   \\
 \null                           & \gon                       & $1.1 \pm 0.17$ & $0.94 \pm 0.15$ &   & & $0.55 \pm 0$ & $5.5 \pm 0$ &   \\
 \hline
 \multirow{4}{1cm}{\textsc{Cover} (z=50)} & \coresetOutliers  & $1 \pm 0$ & $1 \pm 0$ &   & & $1.83 \pm 0.22$ & $5.38 \pm 0.22$ &   \\
 \null                           & \charikar                  & $0.98 \pm 0.09$ & -- &   & & $0.55 \pm 0$ & $5.5 \pm 0$ &   \\
 \null                           & \sampledCharikar           & $0.99 \pm 0.1$ & $1.04 \pm 0.2$ &   & & $0.55 \pm 0$ & $5.5 \pm 0$ &   \\
 \null                           & \gon                       & $1.13 \pm 0.11$ & $1.02 \pm 0.17$ &   & & $0.55 \pm 0$ & $5.5 \pm 0$ &   \\
 \hline
 \multirow{4}{1cm}{\textsc{Higgs+} (z=10)} & \coresetOutliers & $1 \pm 0$ & $1 \pm 0$ & $1 \pm 0$ & & $0.12 \pm 0.01$ & $0.26 \pm 0.02$ & $0.43 \pm 0.02$ \\
 \null                           & \charikar                  & $0.98 \pm 0.18$ & $0.97 \pm 0.04$ & -- & & $0.07 \pm 0$ & $0.7 \pm 0$ & $7 \pm 0$ \\
 \null                           & \sampledCharikar           & $1.09 \pm 0.17$ & $1.61 \pm 0.2$ & $3.03 \pm 0.51$ & & $0.07 \pm 0$ & $0.7 \pm 0$ & $7 \pm 0$ \\
 \null                           & \gon                       & $1.63 \pm 0.29$ & $1.3 \pm 0.16$ & $1.4 \pm 0.02$ & & $0.07 \pm 0$ & $0.7 \pm 0$ & $7 \pm 0$ \\
 \hline
 \multirow{4}{1cm}{\textsc{Cover+} (z=10)} & \coresetOutliers & $1 \pm 0$ & $1 \pm 0$ &   & & $0.65 \pm 0.17$ & $1.93 \pm 0.21$ &   \\
 \null                           & \charikar                  & $0.99 \pm 0.15$ & $1.02 ^*$ &   & & $0.55 \pm 0$ & $5.5 \pm 0$ &   \\
 \null                           & \sampledCharikar           & $0.97 \pm 0.14$ & $0.96 \pm 0.15$ &   & & $0.55 \pm 0$ & $5.5 \pm 0$ &   \\
 \null                           & \gon                       & $3.07 \pm 2.02$ & $1.44 \pm 0.22$ &   & & $0.55 \pm 0$ & $5.5 \pm 0$ &   \\
 \hline
 
 \normalsize
\end{tabular}
\caption{Comparion between clustering radii and between working memory requirements} \label{suppl-tab-1}
\end{table*}

\null
\begin{table*}
\centering
\scriptsize
\begin{tabular}{ c c  c c c  c  c c c}
 \hline 
 \multirow{3}{*}{Dataset} & \multirow{3}{*}{Algorithm} & \multicolumn{3}{c}{Update time (ms)} & \null & \multicolumn{3}{c}{Query time (s)} \\
 \cline{3-5} \cline{7-9}
 \null & \null & \multicolumn{3}{c}{Window size} & \null & \multicolumn{3}{c}{Window size}     \\
 \null & \null & $10^4$ & $10^5$ & $10^6$ &  & $10^4$ & $10^5$ & $10^6$   \\
 \hline  \hline
 \multirow{4}{1cm}{\textsc{Higgs} (z=10)} & \coresetOutliers  & $0.66 \pm 0.17$ & $1 \pm 0.44$ & $1.98 \pm 0.7$ & &  $2.52 \pm 1.22$ & $5.5 \pm 0.87$ & $8.2 \pm 0.42$ \\
 \null                           & \charikar                  & $0 \pm 0$ & $0 \pm 0$ & $0 \pm 0$ & & $24.88 \pm 2.52$ & $3080.8 \pm 160.8$ & -- \\
 \null                           & \sampledCharikar           & $0 \pm 0$ & $0 \pm 0$ & $0 \pm 0$ & & $2.51 \pm 0.26$ & $31.5 \pm 1.34$ & $297.8 \pm 30.9$ \\
 \hline
 \multirow{4}{1cm}{\textsc{Higgs} (z=50)} & \coresetOutliers  & $1.24 \pm 0.66$ & $4.03 \pm 0.53$ & $5.39 \pm 1.14$ & & $10.51 \pm 0.76$ & $93.22 \pm 24.76$ & $75.99 \pm 10.38$ \\
 \null                           & \charikar                  & $0 \pm 0$ & $0 \pm 0$ & $0 \pm 0$ & & $27.97 \pm 2.68$ & -- & -- \\
 \null                           & \sampledCharikar           & $0 \pm 0$ & $0 \pm 0$ & $0 \pm 0$ & & $2.87 \pm 0.24$ & $47.95 \pm 2.11$ & $384.4 \pm 90.92$ \\
 \hline
 \multirow{4}{1cm}{\textsc{Cover} (z=10)} & \coresetOutliers  & $1.05 \pm 0.46$ & $3.74 \pm 0.57$ &   & & $12.85 \pm 2.59$ & $87.82 \pm 46.08$ &   \\
 \null                           & \charikar                  & $0 \pm 0$ & $0 \pm 0$ &   & & $314.93 \pm 55.21$ & $5\cdot10^4 \ ^*$ &   \\
 \null                           & \sampledCharikar           & $0 \pm 0$ & $0 \pm 0$ &   & & $32.1 \pm 5.25$ & $400.81 \pm 66.89$ &   \\
 \hline
 \multirow{4}{1cm}{\textsc{Cover} (z=50)} & \coresetOutliers  & $3.31 \pm 0.8$ & $11.22 \pm 1.28$ &   & & $44.06 \pm 6.88$ & $892.82 \pm 654.3$ &   \\
 \null                           & \charikar                  & $0 \pm 0$ & $0 \pm 0$ &   & & $310.87 \pm 43.6$ & -- &   \\
 \null                           & \sampledCharikar           & $0 \pm 0$ & $0 \pm 0$ &   & & $31.06 \pm 3.05$ & $431.18 \pm 28.42$ &   \\
 \hline
 \multirow{4}{1cm}{\textsc{Higgs+} (z=10)} & \coresetOutliers & $0.65 \pm 0.26$ & $1.22 \pm 0.37$ & $1.89 \pm 0.42$ & & $0.92 \pm 0.93$ & $5.28 \pm 1.82$ & $8.37 \pm 0.31$ \\
 \null                           & \charikar                  & $0 \pm 0$ & $0 \pm 0$ & $0 \pm 0$ & & $27.53 \pm 3.53$ & $3224.2 \pm 252.4$ & -- \\
 \null                           & \sampledCharikar           & $0 \pm 0$ & $0 \pm 0$ & $0 \pm 0$ & & $2.7 \pm 0.22$ & $35.34 \pm 3.14$ & $328.1 \pm 36.08$ \\
 \hline
 \multirow{4}{1cm}{\textsc{Cover+} (z=10)} & \coresetOutliers & $1.21 \pm 0.5$ & $3.14 \pm 0.6$ &   & & $7.08 \pm 6.54$ & $51.22 \pm 39.68$ &   \\
 \null                           & \charikar                  & $0 \pm 0$ & $0 \pm 0$ &   & & $402.32 \pm 71.95$ & $4\cdot10^4 \ ^*$ &   \\
 \null                           & \sampledCharikar           & $0 \pm 0$ & $0 \pm 0$ &   & & $41.47 \pm 9.44$ & $418.03 \pm 47.58$ &   \\
 \hline
 
 \normalsize
\end{tabular}
\caption{Comparison between update times and between query times} \label{suppl-tab-2}
\end{table*}

\begin{figure}[h]
    \centering
    \includegraphics[width=0.5\textwidth]{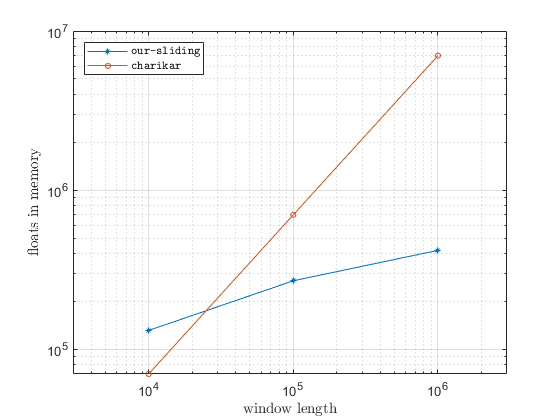}
    \caption{Working memory (\texttt{Higgs}, $z=10$)} \label{fig-memory}
\end{figure}

Table~\ref{suppl-tab-2} reports update times (in milliseconds) and
query times (in seconds).  For \coresetOutliers, the update time is
the time required to process any newly arrived point, while the query
time includes the time to extract the coreset and compute the final
solution on the coreset.  For \charikar\ and \sampledCharikar, the
update time is null, while the query time is the time taken to extract
the solution from the whole window.  The running times reveal that, by
virtue of the coreset-based approach, \coresetOutliers\ features a
query time much smaller than the one of \charikar\ and
\sampledCharikar.  Update times for \coresetOutliers, although not
negligible, are three orders of magnitude smaller than the query
times.  The experiments for $z=50$ show that while the working memory
requirements and running times increase with $z$, for reasonably large
values of $N$, \coresetOutliers\ still exhibits a much lower memory
footprint than \charikar and still returns solutions of comparable
quality.

We also tested the sensitivity of our algorithm's performance to the
$\lambda$ parameter, making it range in $[0,1]$.  The results for
\texttt{Higgs} are shown in Figure~\ref{fig-lambda} and in
Tables~\ref{suppl-tab-3} and \ref{suppl-tab-4}.  Specifically,
Table~\ref{suppl-tab-3} reports on the sensitivity of the clustering
radius and of the memory requirements, while Table~\ref{suppl-tab-4}
reports on the sensitivity of the update and query times.  The
experiments were run on \texttt{Higgs} and \texttt{Cover}, with
$k=z=10$, setting, as before, $\delta = 2/3$, $\beta = 0.5$, $\mindist
= 0.01$, $\maxdist = 10^4$.  For $\lambda$ we used the values: $0,
0.1, 0.5, 1$.  As shown in Figure \ref{fig-lambda}, setting $\lambda =
0$ (i.e., maintaining the full histograms) leads to an unbearable
increase in memory usage, hence in execution times. With approximate
histograms (i.e., $\lambda >0$), the memory usage decreases as
$\lambda$ increases, with a significant drop already for $\lambda =
0.1$. While $\lambda =0$ yields the solution with best approximation,
in our tests we often obtained the same solution using $\lambda =
0.1$. Most importantly, the degradation of the clustering radius never
  exceeded $1\%$, even for $\lambda = 1$.

\begin{table*}
\centering
\scriptsize
\begin{tabular}{ c c  c c c  c  c c c}
 \hline 
 \multirow{3}{*}{Dataset} & \multirow{3}{*}{Algorithm} & \multicolumn{3}{c}{Clustering radius} & \null & \multicolumn{3}{c}{Memory ($\times 10^6$ floats)} \\
 \cline{3-5} \cline{7-9}
 \null & \null & \multicolumn{3}{c}{Window size} & \null & \multicolumn{3}{c}{Window size}     \\
 \null & \null & $10^4$ & $10^5$ & $10^6$ &  & $10^4$ & $10^5$ & $10^6$   \\
 \hline  \hline
 \multirow{4}{1cm}{\textsc{Higgs} (z=10)} & $\lambda=0$  & $2.826 \pm 0.178$ & $4.289 \pm 0.162$ & -- & & $0.53 \pm 0.04$ & $4.85 \pm 0.96$ & -- \\
 \null                           & $\lambda=0.1$         & $2.826 \pm 0.178$ & $4.289 \pm 0.162$ & $5.981 \pm 0.124$ & & $0.15 \pm 0.01$ & $0.36 \pm 0.03$ & $0.6 \pm 0.02$ \\
 \null                           & $\lambda=0.5$         & $2.808 \pm 0.195$ & $4.297 \pm 0.165$ & $6.031 \pm 0.027$ & & $0.13 \pm 0.01$ & $0.27 \pm 0.02$ & $0.42 \pm 0$ \\
 \null                           & $\lambda=1$           & $2.812 \pm 0.178$ & $4.281 \pm 0.176$ & $6.031 \pm 0.021$ & & $0.12 \pm 0.01$ & $0.24 \pm 0.02$ & $0.37 \pm 0$ \\
 \hline
 \multirow{4}{1cm}{\textsc{Cover} (z=10)} & $\lambda=0$  & $1177.56 \pm 400.03$ & $1976.02 \pm 150.66$ &   & & $0.75 \pm 0.17$ & $2.3 \pm 0.15$ &   \\
 \null                           & $\lambda=0.1$         & $1177.56 \pm 400.03$ & $1973.31 \pm 149.53$ &   & & $0.74 \pm 0.17$ & $2.23 \pm 0.14$ &   \\
 \null                           & $\lambda=0.5$         & $1177.56 \pm 400.03$ & $1954.08 \pm 145.71$ &   & & $0.72 \pm 0.18$ & $2.06 \pm 0.13$ &   \\
 \null                           & $\lambda=1$           & $1180.44 \pm 405.86$ & $1967.02 \pm 150.61$ &   & & $0.71 \pm 0.18$ & $2 \pm 0.12$ &   \\
 \hline
\end{tabular}
\caption{Sensitivity of \coresetOutliers\ to $\lambda$: clustering radii and working memory requirements} \label{suppl-tab-3}
\end{table*}

\begin{table*}
\centering
\scriptsize
\begin{tabular}{ c c  c c c  c  c c c}
 \hline 
 \multirow{3}{*}{Dataset} & \multirow{3}{*}{Algorithm} & \multicolumn{3}{c}{Update time (ms)} & \null & \multicolumn{3}{c}{Query time (s)} \\
 \cline{3-5} \cline{7-9}
 \null & \null & \multicolumn{3}{c}{Window size} & \null & \multicolumn{3}{c}{Window size}     \\
 \null & \null & $10^4$ & $10^5$ & $10^6$ &  & $10^4$ & $10^5$ & $10^6$   \\
 \hline  \hline
 \multirow{4}{1cm}{\textsc{Higgs} (z=10)} & $\lambda=0$  & $3.49 \pm 0.69$ & $35.18 \pm 21.11$ & -- & & $2.49 \pm 1.16$ & $5.91 \pm 0.74$ & -- \\
 \null                           & $\lambda=0.1$         & $0.72 \pm 0.38$ & $1.56 \pm 0.53$ & $2.58 \pm 0.41$ & & $2.47 \pm 1.18$ & $5.88 \pm 0.89$ & $8.88 \pm 0.47$ \\
 \null                           & $\lambda=0.2$         & $0.57 \pm 0.12$ & $1.39 \pm 0.42$ & $2.06 \pm 0.67$ & & $2.46 \pm 1.19$ & $5.82 \pm 0.91$ & $7.74 \pm 0.25$ \\
 \null                           & $\lambda=1$           & $0.57 \pm 0.1$ & $1.32 \pm 0.36$ & $1.99 \pm 0.32$ & & $2.47 \pm 1.2$ & $5.81 \pm 0.92$ & $7.77 \pm 0.15$ \\
 \hline
 \multirow{4}{1cm}{\textsc{Cover} (z=10)} & $\lambda=0$  & $1.12 \pm 0.38$ & $3.75 \pm 0.32$ &   & & $12.44 \pm 2.38$ & $90.65 \pm 51.26$ &   \\
 \null                           & $\lambda=0.1$         & $1.33 \pm 0.69$ & $3.72 \pm 0.53$ &   & & $12.31 \pm 1.92$ & $90.36 \pm 51.44$ &   \\
 \null                           & $\lambda=0.2$         & $1.27 \pm 0.5$ & $3.57 \pm 0.25$ &   & & $12.41 \pm 2.19$ & $90 \pm 51.14$ &   \\
 \null                           & $\lambda=1$           & $1.21 \pm 0.5$ & $4.29 \pm 2.28$ &   & & $12.32 \pm 2.17$ & $90.34 \pm 50.61$ &   \\
 \hline
\end{tabular}
\caption{Sensitivity of \coresetOutliers\ to $\lambda$: update and query times} \label{suppl-tab-4}
\end{table*}

\begin{figure}[h]
    \centering
    \includegraphics[width=0.5\textwidth]{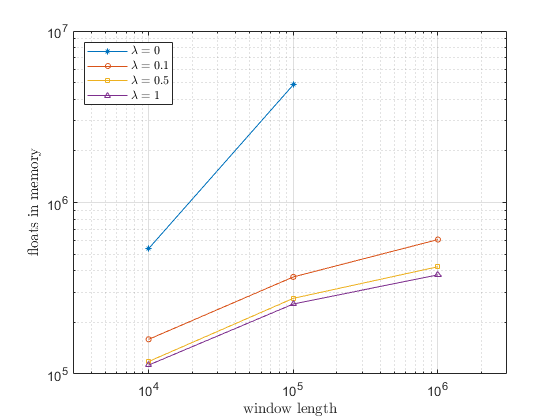}
    \caption{\coresetOutliers: sensitivity to $\lambda$ (\texttt{Higgs}, $z=10$)} \label{fig-lambda}
\end{figure}

Overall, the experiments confirm that \coresetOutliers\ is able to achieve precision comparable to the sequential algorithms at a fraction of their memory/time requirements.

\subsection{Effective diameter}

We compared our algorithm for estimating the effective diameter
described in Section \ref{sec:effdiameter} (dubbed \coresetEf) against
the following sequential baseline (dubbed \seqEf).  \seqEf\ computes
all $N^2$ distances in the window $W$ and, to avoid storing all of
them, only keeps track of  how many distances lay in each
interval $[\mindist\cdot(1+\rho)^i, \mindist\cdot(1+\rho)^{i+1}]$, for $i \geq
0$, by maintaining the appropriate counters. We set $\rho = 0.01$ so
that the error due to this discretization is minimal. After all
distances have been computed, the algorithm sweeps the counters and
returns the minimum value $\mindist\cdot(1+\rho)^i$ for which at least
$\lceil\alpha |W|^2\rceil$ distances fall below
that value.  This same procedure, adapted to account for weights, is also
used in \coresetEf\ to compute the solution on the weighted coreset.

\begin{table*}
\centering
\scriptsize
\begin{tabular}{ c c  c c c  c  c c c}
 \hline 
 \multirow{3}{*}{Dataset} & \multirow{3}{*}{Algorithm} & \multicolumn{3}{c}{Diameter ratio} & \null & \multicolumn{3}{c}{Memory ($\times 10^6$ floats)} \\
 \cline{3-5} \cline{7-9}
 \null & \null & \multicolumn{3}{c}{Window size} & \null & \multicolumn{3}{c}{Window size}     \\
 \null & \null & $10^4$ & $10^5$ & $10^6$ &  & $10^4$ & $10^5$ & $10^6$   \\
 \hline  \hline
 \multirow{2}{*}{\textsc{Higgs-eff}} & $\coresetEf$   & $1 \pm 0$ & $1 \pm 0$ & $1 \pm 0$ & & $0.52 \pm 0.36$ & $0.41 \pm 0.23$ & $0.53 \pm 0.1$ \\
 \null                           & $\seqEf$             & $0.991 \pm 0.019$ & $0.992 \pm 0.009$ & -- & & $0.07 \pm 0$ & $0.7 \pm 0$ & $7 \pm 0$ \\
 \hline
\end{tabular}
\caption{Effective diameter: comparison between  estimates and between working memory requirements} \label{suppl-tab-5}
\end{table*}

\begin{table*}
\centering
\scriptsize
\begin{tabular}{ c c  c c c  c  c c c}
 \hline 
 \multirow{3}{*}{Dataset} & \multirow{3}{*}{Algorithm} & \multicolumn{3}{c}{Update time (ms)} & \null & \multicolumn{3}{c}{Query time (s)} \\
 \cline{3-5} \cline{7-9}
 \null & \null & \multicolumn{3}{c}{Window size} & \null & \multicolumn{3}{c}{Window size}     \\
 \null & \null & $10^4$ & $10^5$ & $10^6$ &  & $10^4$ & $10^5$ & $10^6$   \\
 \hline  \hline
 \multirow{2}{*}{\textsc{Higgs-eff}} & $\coresetEf$   & $3.63 \pm 3.76$ & $2.53 \pm 1.57$ & $3.33 \pm 1.16$ & & $0.05 \pm 0.01$ & $0.77 \pm 0.16$ & $7.41 \pm 0.57$ \\
 \null                           & $\seqEf$             & $0 \pm 0$ & $0 \pm 0$ & $0 \pm 0$ & & $9.26 \pm 1.23$ & $994.81 \pm 38.63$ & -- \\
 \hline
\end{tabular}
\caption{Effective diameter: comparison between update times and between and query times} \label{suppl-tab-6}
\end{table*}

We experimented on the \texttt{Higgs-eff} dataset, which is another
artificially inflated version of the \texttt{Higgs} dataset where a
true outlier (i.e., a random point whose norm is 100 times the diameter of
the original dataset) is injected, on average, every 1000 points.  
In the experiment, we set $\alpha
=0.9$.  Since, for every tested window size $N$, $\alpha N^2 \ll
(N-N/1000)^2$, we expect that, in this controlled experiment, the
$\alpha$-effective diameter of each window of 
\texttt{Higgs-eff} to be close to the
diameter of the non-outlier points in the window. For our algorithm we
set $\epsilon = 5/3$ and $\lambda = \beta = 0.5$. Moreover, we set
$\mindist = 0.01$ and $\maxdist = 10^4$.  Finally, we set $\eta =
1/1000$, which is a very conservative lower bound to the ratio between
the effective diameter and the diameter of the dataset for any window
$W$.

The results of these experiments are reported in
Tables~\ref{suppl-tab-5} and
\ref{suppl-tab-6}. Table~\ref{suppl-tab-5} reports the ratio between
the effective diameter computed by \seqEf\ and the (conservative)
upper estimate $\tilde{\Delta}_{T,W}^{\alpha}$ computed by \coresetEf, as well
as the average number of floats maintained in memory by the
algorithms. As shown in the
table, the solution returned by \coresetEf\ is almost
indistinguishable from the one returned by \seqEf, for those window
sizes for which \seqEf, whose complexity grows quadratically, could be
executed within reasonable times.  On the other hand, the memory usage
of \coresetEf\ grows very slowly with $N$ and thus, for large enough
window sizes, becomes lower than the one of \seqEf. Table
\ref{suppl-tab-6} reports the update and query times for the two
algorithms. Due to the reduced coreset size, the query times of
\coresetEf\ are orders of magnitude lower than those of \seqEf, and
the updated times of \coresetEf, although not negligible, are
significantly smaller than query times.

Finally, we tested on tailor-made artificial datasets the impact of
the parameter $\eta$, which is a (possibly crude) lower bound on the
ratio between the effective diameter and the diameter. In fact, the
theoretical bounds on the coreset size embody a factor proportional to
$(1/\eta)^D$ which could lead to a severe deterioration of the
performance indicators for low (i.e., conservative) values of $\eta$. In
reality, for datasets where the discrepancy between diameter and
effective diameter is caused by few distant outliers (noisy points),
since the balls centered on coreset points have radius
$\BO{\epsilon\eta\Delta_W} = \BO{\epsilon\Delta^{\alpha}_W}$ and since
most of the points will be contained in a ball of radius
$\BO{\Delta^{\alpha}_W}$, with only a few outliers at distance
$\BO{\Delta_W}$, the actual number of points maintained in the coreset
should not really depend on $\eta$.  To test this intuition we created
artificial datasets, by generating random points in a ball of unit
radius with a few outliers (one every 1000 points, on average) on the
surface of a ball of radius $R$, for values of $R$ in $\{10, 100,
1000 \}$.  We set $\alpha = 0.9$ as before, and ran our algorithm with
$\epsilon = 10/3$ and $\lambda = 0.5$. Moreover, we set $\eta = 1/(2R)$
as it lower bounds the ratio between the effective diameter and the
diameter. We report the results for window size $N = 10^5$ since a
similar pattern emerges for other values of $N$.  Indeed, as
Table \ref{suppl-tab-7} show, the memory usage is in practice almost
constant across all values of $\eta$.

\begin{table*}
\centering
\scriptsize
\begin{tabular}{ c c  c c c }
 \hline 
 Dataset & Eff. Diameter & Memory ($\times 10^6$ floats) & Update time (ms) & Query time (s) \\
 \hline
 $R=10$ & $1.175 \pm 0$  & $0.22 \pm 0.05$ &  $0.91 \pm 0.31$ &  $1.69 \pm 0.5$   \\
 $R=100$ & $1.175 \pm 0.006$ & $0.24 \pm 0.11$ & $1.79 \pm 1.01$ & $0.84 \pm 0.25$ \\
 $R=1000$ & $1.178 \pm 0.006$ & $0.35 \pm 0.14$ & $1.9 \pm 1.39$ & $4.2 \pm 1.34$ \\
 \hline
\end{tabular}
\caption{Effective diameter: sensitivity to $\eta$} \label{suppl-tab-7}
\end{table*}


\section{Conclusions} \label{sec:conclusions}

In this paper, we have presented coreset-based streaming algorithms
for the $k$-center problem with $z$ outliers and for the estimation of
the $\alpha$-effective diameter under the sliding window setting. Our
algorithms require working memory considerably smaller than the window
size, and, with respect to the state-of-the-art sequential algorithms
executed on the entire window, they are up to orders of magnitude faster,
while achieving comparable accuracy. The effectivenes of our approach
has been confirmed by a set of proof-of-concept experiments on both
real-world and synthetic datasets.

Based on the theoretical analysis conducted in the paper, the space
and time required by our algorithms to attain high accuracy seem to grow steeply (in fact, exponentially) with the doubling dimension of
the stream. An interesting, yet challenging, research avenue is to
investigate whether this steep dependence can be ameliorated by means
of alternative techniques (e.g., the use of randomization).




\bibliographystyle{plain}
\bibliography{biblio.bib}

\end{document}